\documentclass[a4paper,11pt]{article}

\input{macros.tex}

\title{
\toptitlebar
{{\center\baselineskip 18pt
                      {\Large\bf Non-Euclidean Gradient Descent Operates\\ at the Edge of Stability}\footnote{This work is accepted to International Conference on Machine Learning (ICML 2026) for an oral presentation.}}
} 
\bottomtitlebar}
\date{}
\author[1]{Rustem Islamov}
\author[2]{Michael Crawshaw}
\author[3]{Jeremy Cohen}
\author[3]{Robert Gower}
\affil[1]{University of Basel, Switzerland}
\affil[2]{George Mason University, USA}
\affil[3]{Flatiron Institute, USA}

\begin{document}

\maketitle

\begin{abstract}
The Edge of Stability (\algname{EoS}) is a phenomenon where the sharpness (largest eigenvalue) of the Hessian approaches and then hovers near the stability threshold $2/\eta$ during gradient descent (\algname{GD}) with step size $\eta$. Despite (apparently) violating classical smoothness assumptions, \algname{EoS} has been widely observed in deep learning, but its theoretical foundations remain incomplete. We provide an interpretation of \algname{EoS} through the lens of Directional Smoothness~\cite{mishkin2024directional}. This interpretation naturally extends to non-Euclidean norms, which we use to define generalized sharpness under an arbitrary norm.  Our generalized sharpness measure includes previously studied vanilla \algname{GD} and preconditioned \algname{GD} as special cases, as well as methods for which \algname{EoS} has not been studied, such as  $\ell_{\infty}$-descent, \algname{Block CD}, \algname{Spectral GD}, and their normalized versions. Through experiments on neural networks, we show that non-Euclidean \algname{GD} with our generalized sharpness also exhibits progressive sharpening followed by oscillations around or above the threshold $2/\eta$. Practically, our framework provides a geometry-aware spectral diagnostic that can be applied across a broad class of non-Euclidean gradient methods.
\end{abstract}

\section{Introduction}

In supervised settings, training machine learning models is posed as empirical risk minimization $\min_{\bw\in\R^d} \cL(\bw),$ where $\bw\in\R^d$ are the neural network's parameters, and $\cL(\bw)$ is the full-batch loss, which we assume is bounded below by $\cL^* > -\infty$. In deep learning, $\cL$ is typically nonconvex and highly structured \citep{li2018visualizing,kim2024exploring}. Nevertheless, first-order methods such as \algname{SGD} and its adaptive variants \citep{duchi2011adaptive,kingma2014adam} are the workhorses of practice and scale effectively to large models, despite a limited theoretical understanding of their success.

Full-batch gradient descent (\algname{GD}) serves as the canonical proxy for analyzing gradient-based training. Classical results for $L$-smooth convex objectives guarantee descent for step sizes up to $2/L$. In contrast, recent empirical work reveals a characteristic two-phase behavior when deep networks are trained with \algname{GD}. In the initial phase, called the progressive sharpening phase, the loss $\cL(\bw_t)$ decreases monotonically while the sharpness $S(\bw_t)\coloneqq \lambda_{\max}(\nabla^2\cL(\bw_t))$ grows. This is followed by the edge-of-stability (\algname{EoS}) phase, where the loss behaves non-monotonically yet decreases over longer horizons, while the sharpness hovers near the threshold $2/\eta$ \citep{cohen2021edgeofstability}.

The \algname{EoS} phenomenon has been found to extend beyond vanilla \algname{GD}.  \citet{cohen2022adaptive} showed that adaptive preconditioning methods such as  \algname{AdaGrad} and \algname{Adam} exhibit an \algname{EoS} characterization that revolves around the top eigenvalue of the \emph{preconditioned} Hessian, while \citet{long2024sharpness} showed that \algname{SAM} obeys a certain \algname{EoS} characterization as well.  Despite these advances, the question of how \algname{EoS} generalizes to other optimizers remains underexplored.  Here we investigate how the \algname{EoS} phenomenon carries over to a broad family of optimization algorithms: that of non-Euclidean gradient descent with respect to an arbitrary norm.
\begin{definition}\label{def:gradmethod}
    For a norm $\norm{\cdot}$   and a step-size $\eta > 0$, the associated non-Euclidean \algname{GD} method is given by the minimization of the regularized linearization around the current point $\bw_t$:
    \begin{align}
      \bw_{t+1} & \in \argmin{\by}\dotprod{\nabla \cL(\bw_t), \by-\bw_t} + \frac{1}{2\eta}\|\by -\bw_t\|^2  \nonumber \\
      &= \bw_t -  \eta\| \nabla \cL(\bw_t)\|_*(\nabla \cL(\bw_t))_*,
      \label{eq:gradmethod}
    \end{align}
    where the \emph{dual norm} $\| \nabla \cL(\bw_t) \|_*$ and \emph{dual gradient} $(\nabla \cL(\bw_t))_*$ are defined as:
    \begin{align}\label{def:dualnorm_dualgrad}
&\| \nabla \cL(\bw_t)\|_* \coloneqq \max_{\|\by\|= 1}\dotprod{ \nabla \cL(\bw_t), \by}, \\
    &(\nabla \cL(\bw_t) )_* \coloneqq \argmax_{\|\by\|= 1}\dotprod{ \nabla \cL(\bw_t), \by}. \nonumber
\end{align}
    We let $\bd_t \coloneqq \| \nabla \cL(\bw_t)\|_*(\nabla \cL(\bw_t))_*$ denote the update ``direction'' (i.e. the update without the step-size $\eta$).
\end{definition}
This formulation reduces to vanilla \algname{GD} when the norm $\norm{\cdot}$ is taken to be the $\ell_2$ norm.  It also subsumes methods not previously studied by prior work on \algname{EoS} such as $\ell_{\infty}$-descent (for $\norm{\cdot}=\ell_{\infty}$) and \algname{Spectral GD} (for $\norm{\cdot}=\norm{\cdot}_{2\rightarrow2}$) \citep{carlson15stochastic} (which underlies the popular \algname{Muon} method \citep{jordan2024muon}), as well as \algname{Block CD} \citep{nesterov2012efficiency} and other coordinate descent variants.

Sometimes, the dual norm is omitted from the update \eqref{eq:gradmethod}.
We refer to the resulting algorithm as \emph{normalized} non-Euclidean \algname{GD}\footnote{For the $\ell_\infty$ norm, \Cref{def:gradmethod} gives $\ell_\infty$-descent, while \Cref{def:normalized_gradmethod} gives normalized $\ell_\infty$-descent, also known as SignGD.}:
\begin{definition}\label{def:normalized_gradmethod}
    For a norm $\norm{\cdot}$ (not necessarily the $\ell_2$ norm) and a step-size $\eta > 0$, the associated \emph{normalized} non-Euclidean \algname{GD} method is given by 
    \begin{equation}
        \bw_{t+1} = \bw_t - \eta(\nabla \cL(\bw_t))_*,  \label{eq:normalized-gradmethod}
    \end{equation}
    where the dual gradient $(\nabla\cL(\bw_t))_*$ is defined in \eqref{def:dualnorm_dualgrad}. In this case, $\bd_t \coloneqq (\nabla \cL(\bw_t))_*.$
\end{definition}
When $\norm{\cdot}$ is the $\ell_\infty$ norm, this formulation recovers normalized $\ell_\infty$-descent (also known as \algname{SignGD} \citep{bernstein2018signsgd}), and when $\norm{\cdot}$ is the spectral norm $\norm{\cdot}_{2\rightarrow2}$, it recovers normalized \algname{Spectral GD} \citep{boyd2004convex} (used in modern optimizers such as \algname{Muon} \citep{jordan2024muon} and \algname{Scion} \citep{pethick2025training}). Our main contributions are:
\begin{enumerate}

\item We identify that an intermediary quantity called directional smoothness $D^{\norm{\cdot}}(\bw,\by)$ \citep{mishkin2024directional} can be used to study the dynamics of sharpness and the \algname{EoS}. Directional smoothness is the average curvature between two consecutive iterates.
\item Through a simple identity, we show that if the loss decreases, then 
 directional smoothness \emph{must} be less than $2/\eta$. If the loss oscillates, then directional smoothness \emph{must} oscillate around $2/\eta$.

\item Extending directional smoothness beyond Euclidean norm, we define a generalized sharpness $S^{\norm{\cdot}}$ of \algname{GD} under any norm $\norm{\cdot}$. In the special cases of Euclidean and preconditioned \algname{GD}, this measure recovers previously established notions of sharpness.

\item Across MLPs, CNNs, and Transformer architectures, we observe that $S^{\norm{\cdot}}$ sharpens and hovers near, and sometimes slightly above, the stability threshold $2/\eta$, providing empirical evidence for geometry-aware \algname{EoS} behavior.

\item To shed light on the mechanism underlying this behavior, we analyze the dynamics of non-Euclidean \algname{GD} on quadratic objectives.

\end{enumerate}

\subsection{Related Works}

The \algname{EoS} phenomenon was first documented for vanilla \algname{GD} with step-size $\eta$, where the sharpness (the maximum Hessian eigenvalue) was observed to hover near the stability threshold $2/\eta$ \citep{cohen2021edgeofstability,wu2018sgd}. This initial work also extended empirical observations to \algname{GD} with momentum and provided intuition for \algname{EoS} on quadratic objectives. Building on this, \citet{arora2022understanding} gave a mathematical analysis of the implicit regularization that arises at \algname{EoS}, showing that in non-smooth loss landscapes the updates of normalized \algname{GD} follow a deterministic flow constrained to the manifold of minimal loss. A subsequent study by \citet{song2023trajectory} demonstrated empirically that \algname{GD} trajectories align with a universal bifurcation diagram during \algname{EoS}, while \citet{damian2022self} identified self-stabilization as the key mechanism: a cubic term in the Taylor expansion along the top Hessian eigenvector introduces negative feedback that drives sharpness back toward $2/\eta$ whenever it exceeds the threshold. Beyond the stability plateau, \citet{ghosh2025learning} analyzed loss oscillations in deep linear networks, demonstrating that they happen in a low-dimensional subspace whose dimension depends on the step-size $\eta$. Finally, several works connect \algname{EoS} with the catapult mechanism observed in training with a large learning rate \citep{lewkowycz2020large, zhu2024catapults, kalra2023phase}.

The phenomenon has also been studied for preconditioned and adaptive methods. \citet{cohen2022adaptive} showed that the sharpness of the preconditioned Hessian stabilizes at the same threshold for methods such as \algname{AdaGrad} and \algname{RMSprop}. Meanwhile, \citet{long2024sharpness} conducted a stability analysis of \algname{SAM} \citep{foret2020sharpness} on quadratics, empirically showing that \algname{SAM} operates at the edge of stability. Extensions beyond full-batch \algname{GD} include \citet{lee2023new}, who analyzed the interaction between batch-gradient distributions and loss geometry to extend \algname{EoS} to \algname{SGD}, and \citet{andreyev2024edge}, who proposed an alternative stochastic counterpart of \algname{EoS}.

Despite this progress, most prior studies have focused on a narrow family of algorithms (e.g., vanilla \algname{GD}, preconditioned \algname{GD}, or \algname{SAM}), leaving a fundamental gap in our understanding of spectral properties and raising the question of whether these insights extend to substantially different optimization methods such as \algname{Muon} \citep{jordan2024muon}, \algname{Scion} \citep{pethick2025training}, and \algname{SignGD} \citep{bernstein2018signsgd}. 
Here, we take a step toward a unifying definition of sharpness for non-Euclidean gradient methods and empirically demonstrate \algname{EoS} behavior across several representative geometries, including $\ell_\infty$-descent, \algname{Block CD}, \algname{Spectral GD}, and their normalized variants.

\section{Progressive Sharpening and Directional Smoothness}

Classical descent guarantees for \algname{GD} rely on global $L$-smoothness, but such bounds are often too pessimistic for neural networks \citep{zhang2019gradient,alimisis2025why}. Instead, we adopt a local, trajectory-aware notion of directional smoothness \citep{mishkin2024directional}.

\begin{definition}
Let $\Delta\cL_t \coloneqq \cL(\bw_{t+1}) -\cL(\bw_t)$ be the change of the function value.
    We call a function $D^{\norm{\cdot}}(\bw_{t},\bw_{t+1})$ a valid \emph{directional smoothness} at iteration $t$ if
\begin{align}
 \Delta\cL_t &\le 
 \langle \nabla \cL(\bw_t),\bw_{t+1}-\bw_t  \rangle + \frac{D^{\norm{\cdot}}(\bw_t,\bw_{t+1})\|\bw_{t+1} - \bw_t\|^2}{2}, \label{eq:quadratic-bound}
\end{align}
where $D^{\norm{\cdot}}(\bw_{t}, \bw_{t+1})$ depends only on the behavior of the loss $\cL$ along the chord $[\bw_t, \bw_{t+1}]$.
\end{definition}
\citet{mishkin2024directional} provide several examples of the directional smoothness. Here, we choose the tightest one
\begin{equation}\label{eq:dirsmooth-opt}
D^{\norm{\cdot}}(\bw,\by)
\coloneqq
\frac{\cL(\by)-\cL(\bw)-\langle \nabla \cL(\bw), \by-\bw\rangle}{\frac{1}{2}\|\by-\bw\|^2},
\end{equation}
which makes \eqref{eq:quadratic-bound} hold with equality. Although this quantity might not be positive (and thus falls outside the positivity requirements of \citet{mishkin2024directional}), positivity is not required in the following presentation. Substituting one step of non-Euclidean \algname{GD} into \eqref{eq:quadratic-bound} yields
\begin{align} 
    \Delta\cL_t &= - \eta \dotprod{\nabla \cL(\bw_{t}),
    \bd_t} + \eta^2 \frac{D^{\norm{\cdot}}(\bw_{t}, \bw_{t+1}) }{2}\|\bd_t\|^2\nonumber \\
     &= - \eta \left( 1-\frac{\eta}{2}D^{\norm{\cdot}}(\bw_{t}, \bw_{t+1})\right)\| \nabla \cL(\bw_t)\|_*^2, \label{eq:quadratic-bound-gd}
\end{align}
where we used that~\eqref{def:dualnorm_dualgrad} implies $$\dotprod{\nabla \cL(\bw_{t}), \bd_t} = \| \nabla \cL(\bw_t)\|_*^2.$$
Consequently, 
whenever $\|\nabla \cL(\bw_t)\|_* > 0$, the loss decreases \emph{iff}
\begin{equation}\label{eq:D-less-than-2-eta}
\Delta\cL_t \le 0
\quad\Longleftrightarrow\quad
D^{\norm{\cdot}}(\bw_t,\bw_{t+1}) \le \frac{2}{\eta}.
\end{equation}

The equivalence in~\eqref{eq:D-less-than-2-eta} justifies the progressive sharpening of the directional smoothness. Note that in deep learning experiments where \algname{EoS}~is observed, the gradient norm remains non-zero \citep{defazio2023optimal,defazio2025gradients}, see the Gradient Norm panel in Figure~\ref{fig:gd_all}. Therefore, according to~\eqref{eq:D-less-than-2-eta}, if the loss initially decreases and then starts to oscillate, as is often observed in training, then directional smoothness must start below $2/\eta$ and then increase (sharpen) up to $2/\eta$, and then oscillate around $2/\eta$. Indeed, see the Directional Smoothness panel in Figure~\ref{fig:gd_all}, where we can see that the directional smoothness progressively sharpens up to $\nicefrac{2}{\eta}$. Thus, almost by definition, directional 
smoothness exhibits the sharpening and \algname{EoS} phase.

\begin{figure*}
    \centering
    \begin{tabular}{c}
         \includegraphics[width=1\linewidth]{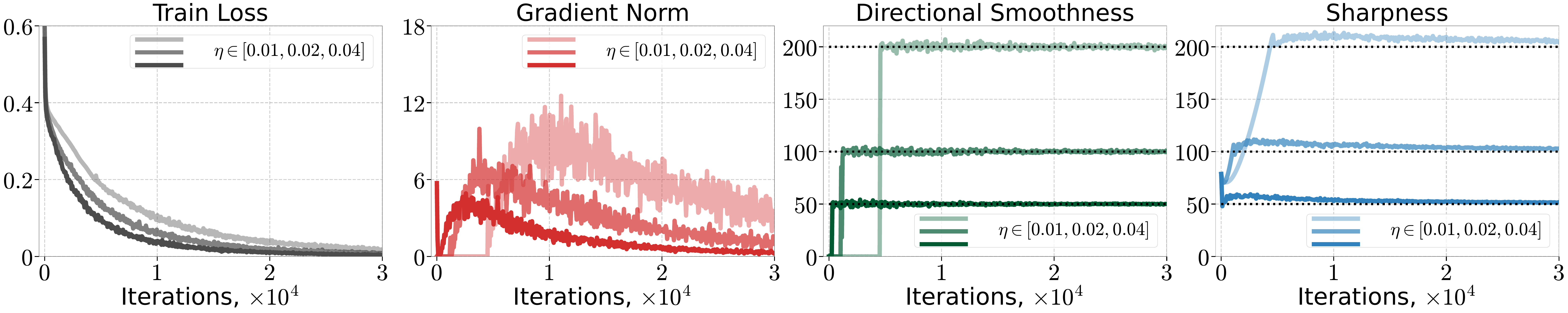}\\
         \includegraphics[width=1\linewidth]{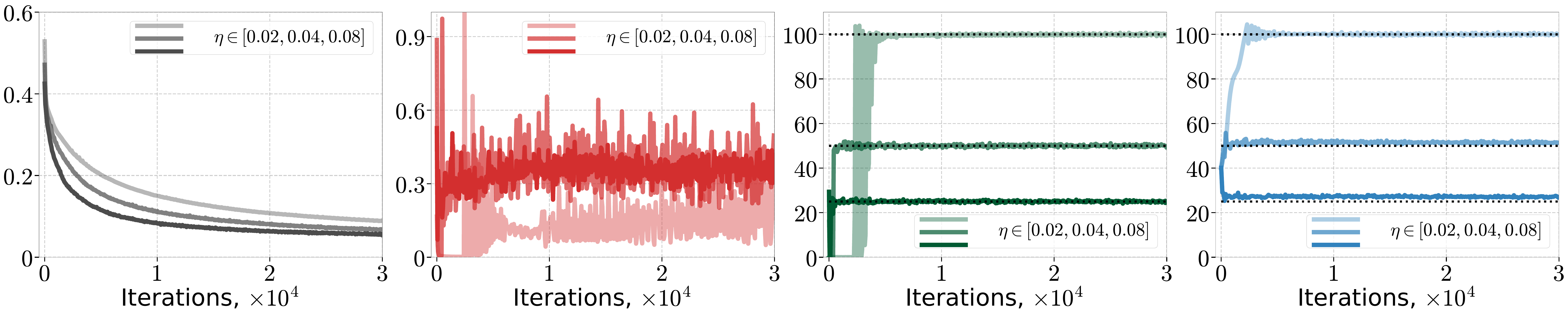}
    \end{tabular}
    \caption{Vanilla \algname{GD} on CIFAR-10-5k. {\bf Top}: MLP; {\bf bottom}: CNN. Columns show train loss, gradient norm, directional smoothness, and standard $\ell_2$ sharpness. Dashed lines mark the stability threshold $2/\eta$.}
    \label{fig:gd_all}
\end{figure*}

\subsection{Connection to Sharpness}\label{sec:direct_smooth_and_sharpness}

Next, we show how directional smoothness is closely related to a Hessian quantity that we will call the generalized sharpness. Assume that $\cL$ is twice continuously differentiable and that $\bd_t\neq0$. We can relate~\eqref{eq:dirsmooth-opt} to sharpness by plugging in one step of non-Euclidean \algname{GD} \eqref{eq:gradmethod} into~\eqref{eq:dirsmooth-opt} and using Taylor's theorem with the Lagrange form of the remainder:
\begin{align}
  D^{\norm{\cdot}}(\bw_{t},\bw_{t+1}) & \coloneqq \frac{  \cL(\bw_{t+1})  - \cL(\bw_{t})  - 
  \dotprod{\nabla \cL(\bw_{t}),-\eta\bd_t}
  }{\frac{1}{2}\|\bw_{t+1}-\bw_{t}\|^2} =  \frac{\bd_t^\top \nabla^2  \cL(\bw_t - \xi_t\eta\bd_t)  \, \bd_t }{\|\bd_t\|^2} \label{eq:D-as-Hess},
\end{align}
where $\xi_t \in (0, \,1)$.
We can further upper-bound \eqref{eq:D-as-Hess} by taking the maximum over all directions
 \begin{align} \label{eq:T2-d-S}
  D^{\norm{\cdot}}(\bw_{t},\bw_{t+1})      
  &\leq \max_{\bd\neq0 }  \frac{\bd^\top \nabla^2  \cL(\bw_t - \xi_t\eta\bd_t )\bd }{\|\bd\|^2} 
\end{align}
If we further assume that the Hessian is almost constant over the line segment $\{ \bx : \bx= \bw_t -\eta \xi\bd_t, \xi \in [0,\;1]\},$ we arrive at the following definition of generalized sharpness:

\begin{definition}
    For any norm $\norm{\cdot}$, we define the \emph{generalized sharpness} as:
    \begin{align}\label{eq:sharpness-norm}
        S^{\norm{\cdot}}(\bw) &\coloneqq \max_{\bd\neq0}  \frac{\bd^\top \nabla^2  \cL(\bw)\bd }{\|\bd\|^2} \\
        &= \max_{\|\bd\|=1}  \bd^\top \nabla^2  \cL(\bw)\bd.\label{eq:sharpness-norm-unit-ball}
\end{align}
\end{definition}

\begin{wrapfigure}{r}{0.53\textwidth}
\hfill
\vspace{-8mm}
\begin{minipage}{\linewidth}
\begin{algorithm}[H]

\begin{algorithmic}[1]
\caption{Frank-Wolfe to approximate $(\star)$}
    \label{alg:FrankWolfe}
     \State {\bf Input:} norm $\norm{\cdot}$, $\gamma_k = \frac{2}{2+k}$, $S_0=0$
       
    \For{restart $m=1, \dots, M$}
        \State$\bd_0 \sim \cN(0,\mI), \; \bd_0 = \Pi_{\norm{\cdot}= 1}(\bd_0)$
        
        \For{$k=0, 1,\dots, K-1$}
            \State $\bv_k = \mathrm{argmax}_{\|\bv\|\le 1}\left<\nabla^2\cL(\bw_t)\bd_k,\bv\right>$

            \State $\bd_{k+1} = (1-\gamma_k)\bd_k + \gamma_k\bv_k$
     \EndFor
      \State $\textcolor{orange}{\bu_K = \Pi_{\norm{\cdot}=1}(\bd_K)}$
     \State $\hat{S}_m = \bd_K^\top\nabla^2\cL(\bw_t)\bd_K$
     
        \State  $S_m = \max\{S_{m-1}, \hat{S}_m\}$
     \EndFor

    \State {\bf Return:} $S_M$
\end{algorithmic}	
\end{algorithm}
\end{minipage}
\end{wrapfigure}

The optimization problem \eqref{eq:sharpness-norm-unit-ball} involves \emph{maximizing} a quadratic function over a nonconvex unit sphere, and is thus challenging to solve in general.
For some choices of norm, the problem \eqref{eq:sharpness-norm} has an analytical solution (e.g., vanilla \algname{GD} or \algname{Block CD}). For other norms, we heuristically approximate the solution to \eqref{eq:sharpness-norm-unit-ball} by solving a relaxed problem 
\[
\max_{\|\bd\|\le 1}\bd^\top\nabla^2\cL(\bw)\bd,
\]
using the Frank-Wolfe (\algname{FW}) algorithm \citep{frank1956algorithm} run from multiple random restarts (Algorithm~\ref{alg:FrankWolfe}). On smooth, non-convex objectives, \algname{FW} is known to converge to a first-order stationary point over convex sets \citep{lacoste2016convergence}.

Since a stationary point is not necessarily the global maximum, we repeatedly run Frank-Wolfe from multiple random restarts and then take the maximum over all trials. Unless a closed form is available, the generalized sharpness reported in the experiments should therefore be interpreted as a multi-restart \algname{FW} estimate. Empirically, we usually observe that the generalized sharpness estimated using this procedure converges to some limiting value as the number of random restarts grows. The true global maximizer does not necessarily lie on the boundary, but using the projection step onto the unit norm sphere in Algorithm~\ref{alg:FrankWolfe} always improved the estimation in our experiments. See Appendix~\ref{sec:frank_wolfe} for a more detailed discussion of our procedure for approximating \eqref{eq:sharpness-norm}.

\section{Examples of Non-Euclidean Gradient Descent}
We begin by showing that the generalized sharpness \eqref{eq:sharpness-norm} recovers previously derived notions of sharpness, establishing consistency with known \algname{EoS} characterizations. We then examine generalized sharpness under several non-Euclidean norms. All experiments are full-batch, use MSE loss for MLP/CNN models and the CE loss for Transformers, and report full-batch Hessian-vector-product estimates of directional smoothness and generalized sharpness; architecture and implementation details are deferred to the appendix.

\textbf{Euclidean $\ell_2$ Norm.} We consider a standard Euclidean $\ell_2$ norm. In this case, the sharpness measure \eqref{eq:sharpness-norm} can be computed explicitly. Indeed, the maximum in \eqref{eq:sharpness-norm} equals the largest eigenvalue of the Hessian $\lambda_{\max}(\nabla^2\cL(\bw_t))$. This result coincides with the sharpness measure introduced in \citet{cohen2021edgeofstability}. In Figure~\ref{fig:gd_all}, we report the training dynamics of vanilla \algname{GD}, flattening all parameters of the networks. We observe that the directional smoothness and sharpness hover at $\nicefrac{2}{\eta}$ when the algorithm enters \algname{EoS} stage, supporting our claims in \eqref{eq:D-less-than-2-eta}.

\textbf{Preconditioned $\ell_2$ Norm. } Let $\mP_t \in \R^{d\times d}$ be a symmetric positive definite matrix, which we will use as a preconditioner.
That is, we define the preconditioned $\ell_2$ norm (also referred to as the Mahalanobis distance) by $\| \bw \|_{\mP_t}^2 \coloneqq  \dotprod{\mP_t \bw, \bw}  \; = \; \| \mP_t^{1/2} \bw\|_2^2.$ Under this norm, preconditioned \algname{GD}~\eqref{eq:gradmethod} is given by
\begin{equation}
  \bw_{t+1} =  \bw_t - \eta\mP^{-1}_t \nabla \cL(\bw_t).
\end{equation}

This case includes \algname{AdaGrad}~\citep{duchi2011adaptive}, \algname{RMSprop}~\citep{tieleman2012rmsprop} and Newton's method as special cases.
According to~\eqref{eq:sharpness-norm}, the correct notion of sharpness for this norm is given by
\begin{align}
    S^{\norm{\cdot}_{\mP_t}}(\bw) &\coloneqq \max_{\bd\neq0}  \frac{\bd^\top \nabla^2  \cL(\bw )\bd }{\|\bd\|_{\mP_t}^2} 
    = \max_{\bv\neq0}  \frac{\bv^\top \mP_t^{-1/2} \nabla^2  \cL(\bw ) \mP_t^{-1/2} \bv }{\|\bv\|_{2}^2},\notag
\end{align}
where we arrived at last equality by using the change of variables $\bv = \mP_t^{1/2} \bd.$ This definition matches the preconditioned-Hessian sharpness used for adaptive methods by \citep{cohen2025centralflow}.

\begin{figure*}[!t]

    \centering
    \begin{tabular}{c}
         \hspace{-3mm}\includegraphics[width=1\linewidth]{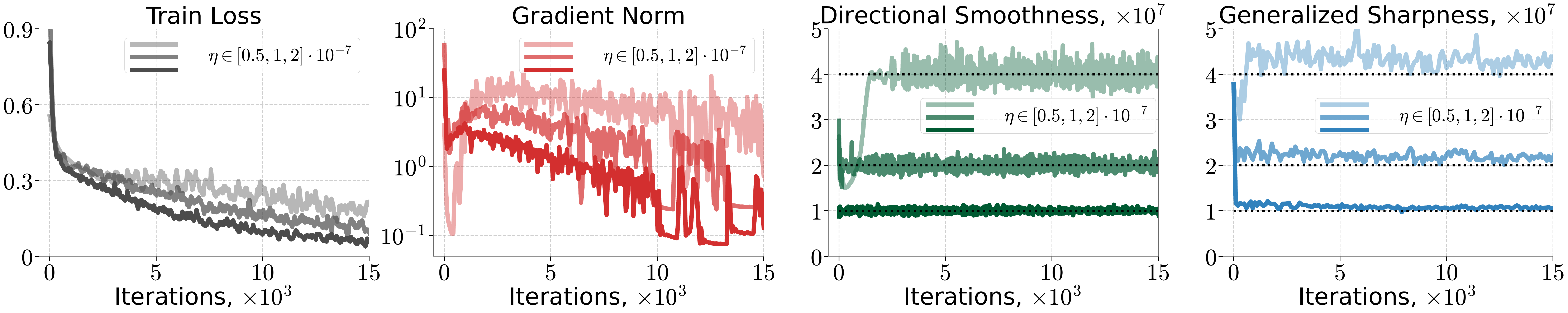} \\
         \includegraphics[width=1\linewidth]{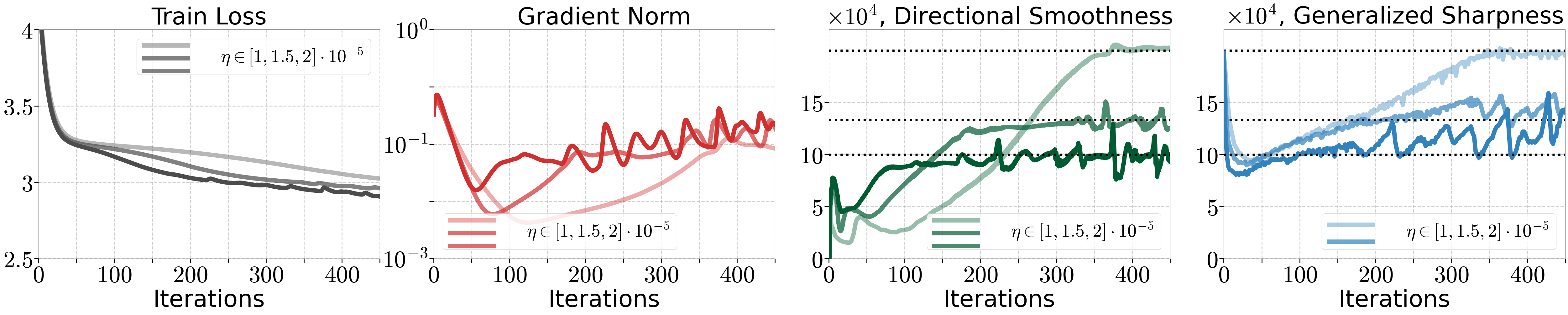} 
    \end{tabular}
    \caption{$\ell_{\infty}$-descent on {\bf Top}: MLP CIFAR-10-5k; {\bf bottom}: Transformer on Tiny Shakespeare. Columns show train loss, gradient norm, directional smoothness, and generalized sharpness \eqref{eq:max-inner-norm}. Dashed lines mark the stability threshold $2/\eta$. }
    \label{fig:signgd_all}
\end{figure*}

\begin{figure*}[!t]
    \centering
    \begin{tabular}{c}
         \includegraphics[width=1\linewidth]{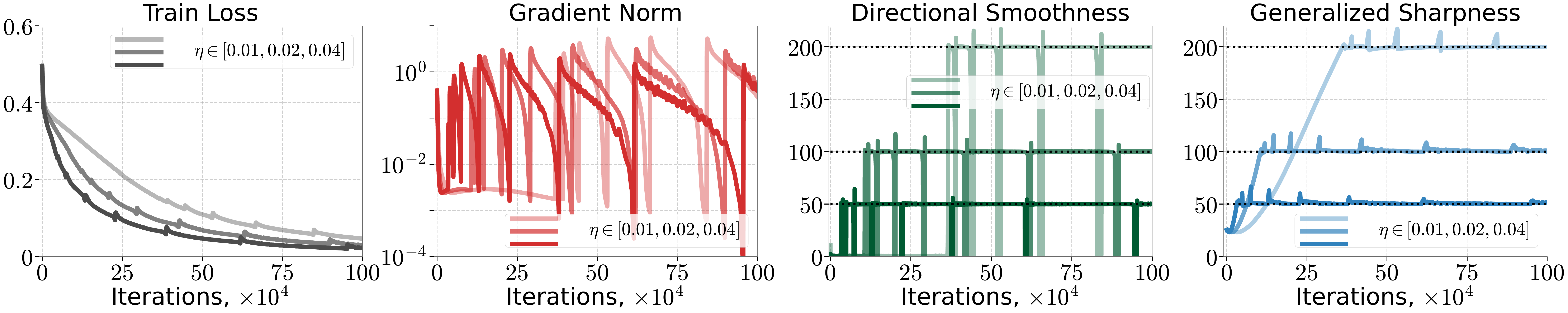} \\
         \includegraphics[width=1\linewidth]{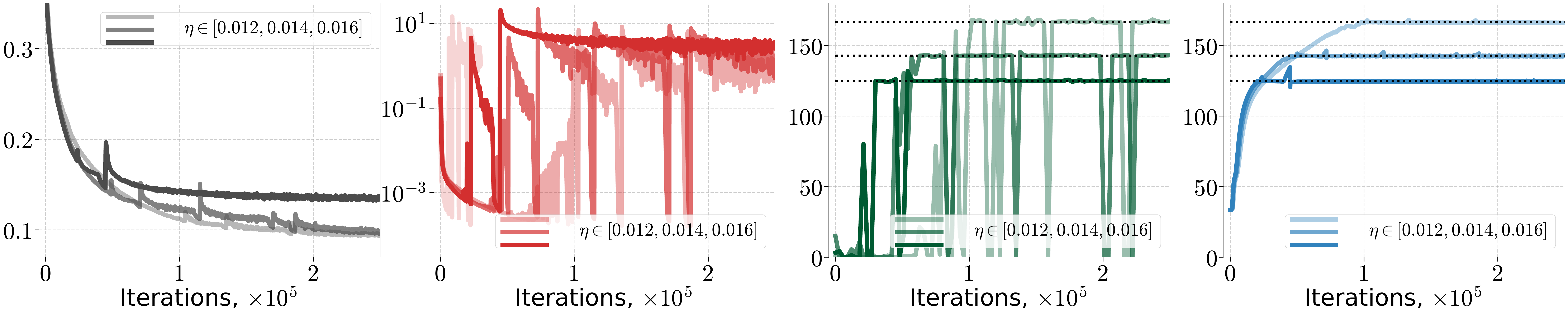}
    \end{tabular}
    \caption{\algname{Block CD} on CIFAR-10-5k. {\bf Top}: MLP, {\bf bottom}: CNN. Columns show train loss, gradient norm, directional smoothness, and generalized sharpness \eqref{eq:sharpness-12norm}. Dashed lines mark the stability threshold $2/\eta$.}
    \label{fig:blockgd_all}
\end{figure*}

\paragraph{Infinity $\ell_{\infty}$ Norm.}
Here we consider the infinity norm over the parameters of the neural network, that is $\norm{\bw}_{\infty} \coloneqq \max_{j\in[d]} |\bw_{j}|.$ The resulting method~\eqref{eq:gradmethod} is the following variant of $\ell_{\infty}$-descent  given by
\begin{equation}
  \bw_{t+1} 
  = \bw_t -  \eta \| \nabla \cL(\bw_t)\|_1\mbox{sign}(\nabla \cL(\bw_t)).
\end{equation}
The corresponding definition of sharpness~\eqref{eq:sharpness-norm} under this norm is given by
\begin{align}\label{eq:max-inner-norm}
S^{\norm{\cdot}_{\infty}}(\bw) 
&=  \max_{\bd}   \bd^\top \nabla^2  \cL(\bw)\bd\quad \mbox{s.t. } \| \bd\|_{\infty} = 1.
\end{align}
This optimization problem~\eqref{eq:max-inner-norm} has also appeared in statistical physics, where it is equivalent to finding the maximum energy—or, correspondingly, the \emph{ground state} in a \emph{flipped sign} formulation—of an Ising spin glass on the hypercube.  For quadratic Hessian, this corresponds to maximizing the Hamiltonian over binary spin assignments \(d_i = \pm1\).  The problem is known to be NP-hard in general \citep{zhang2025ising3d, kochenberger2014qubo}. Therefore, we use Algorithm~\ref{alg:FrankWolfe} to approximate \eqref{eq:max-inner-norm}, with the projection operator being $\Pi_{\norm{\cdot}_{\infty}=1}(\cdot) \equiv {\rm sign}(\cdot)$.

Figure~\ref{fig:signgd_all} presents the convergence results of $\ell_{\infty}$-descent, applied to the flattened networks' parameters. In this case, directional smoothness plateaus at $\nicefrac{2}{\eta}.$ A similar behavior appears for generalized sharpness. We observe several interesting phenomena. First, in some cases, the generalized sharpness hovers \emph{slightly above} the stability threshold $\nicefrac{2}{\eta}$.  As we review in Appendix~\ref{sec:gap_generalized_sharpness_and_eta}, a similar effect has been observed for Euclidean \algname{GD} when there are multiple Hessian eigenvalues at the edge of stability, and we hypothesize this behavior could have a similar origin.
Second, \algname{FW} requires a sufficient number of restarts to obtain a good approximation of the generalized sharpness in \eqref{eq:max-inner-norm}: see Figure~\ref{fig:signgd_FW_restarts}, because the \algname{FW} estimate can underestimate the global maximum when too few restarts are used. In Figure~\ref{fig:signgd_gener_sharpness_vs_l2_sharpness}, we observe on the full CIFAR-10 dataset that the generalized sharpness stabilizes near $2/\eta$, while the $\ell_2$ sharpness remains well below. This shows that the \algname{EoS} threshold is not captured by the standard $\ell_2$ sharpness in this setting, but is captured by the proposed generalized sharpness.

\textbf{Block $\ell_{1,2}$ Norm.} In this case, we take into account the block-wise structure of neural networks. Let the parameters $\bw$ be split into $L$ blocks, i.e., $\bw = (\bw^1,\ldots, \bw^L) \in \R^{d_1} \oplus \R^{d_2} \ldots \oplus \R^{d_L}$ where $\sum_{\ell=1}^L d_\ell = d.$ We consider \algname{GD} in the $\norm{\cdot}_{1,2}$ norm\footnote{In this case, each block $\bw^{\ell}$ is treated as a vector.} defined as $\|\bw\|_{1,2} \coloneqq \sum_{\ell=1}^L\|\bw^\ell\|_2.$
Let $\ell_{\max} \coloneqq {\rm argmax}_{\ell\in[L]}\|\nabla_{\bw^\ell} \cL(\bw_t)\|$. Then \algname{GD} in this norm reduces to \algname{Block CD}
\begin{align}
\bw_{t+1}^{\ell_{\max}} &= \bw_t^{\ell_{\max}} - \eta\nabla_{\bw^{\ell_{\max}}} \cL (\bw_t),\\ 
\bw_{t+1}^{\ell} &= \bw_t^{\ell} \quad \mbox{for } \ell \neq \ell_{\max}.\notag
\end{align}
The derivations of \algname{GD} in this norm are given in Lemma~\ref{lem:block_gd}. The corresponding definition of  sharpness~\eqref{eq:sharpness-norm} under this norm is given by
\begin{align} \label{eq:sharpness-12norm}
\hspace{-0.2cm}S^{\norm{\cdot}_{1,2}}(\bw_t) 
=  \max_{\bd }    \dotprod{\bd, \nabla^2  \cL(\bw_t )\bd } \mbox{        s.t. } \| \bd\|_{1,2} = 1.
\end{align}
The solution to \eqref{eq:sharpness-12norm} can be given explicitly if the Hessian $\nabla^2\cL(\bw_t)$ is PSD (see Lemma~\ref{lem:sharpness_l12_norm})
\begin{equation}
S^{\norm{\cdot}_{1,2}}(\bw) = \max_{\ell\in[L]}\lambda_{\max}(\nabla^2_{\bw^\ell}\cL(\bw)).
\end{equation}
However, for the general $\nabla^2\cL(\bw_t)$, solving \eqref{eq:sharpness-12norm} is NP-hard \citep{bhattiprolu2021framework}, but still can be approximated by the \algname{FW} algorithm. The exact steps of \algname{FW} in this case are derived in Lemma~\ref{lem:FW_l12_norm}.

Figure \ref{fig:blockgd_all} shows the convergence of \algname{Block CD}, where we adopt the natural block-wise structure of the network -- each block corresponding to a weight matrix or bias vector of a layer.
The generalized sharpness, which is approximated by the maximum eigenvalue of each block of the Hessian, approaches the threshold $\nicefrac{2}{\eta}$, supporting our theoretical observations. In contrast, the directional smoothness curves display sharper dynamics: while they also reach $\nicefrac{2}{\eta}$, they exhibit sudden drops whenever training shifts from a layer already at the \algname{EoS} regime to one that has not yet reached it. These drops are also mirrored in the gradient norm dynamics. Similar to $\ell_{\infty}$, \algname{FW} algorithm is sensitive to the number of restarts $M$. Figure~\ref{fig:blockgd_true_vs_fw} reports that \algname{FW} with $M=10$ provides a stable estimation of the generalized sharpness, while \algname{FW} with $M=1$ does not.

\textbf{Spectral $\norm{\cdot}_{2\rightarrow2}$ Norm.} To handle matrix norms, we shift perspective and treat the layers of the network as blocks of matrices\footnote{We use upper case notation to highlight the matrix structure.} $\mW \coloneqq  (\mW^1, \ldots, \mW^L)$. In this setting, the natural inner product is the matrix trace $\dotprod{\mW, \mG} \coloneqq \trace{\mW^\top \mG}$. 
\begin{figure*}[!t]
    \centering
    \begin{tabular}{c}
         \hspace{-3mm}\includegraphics[width=1\linewidth]{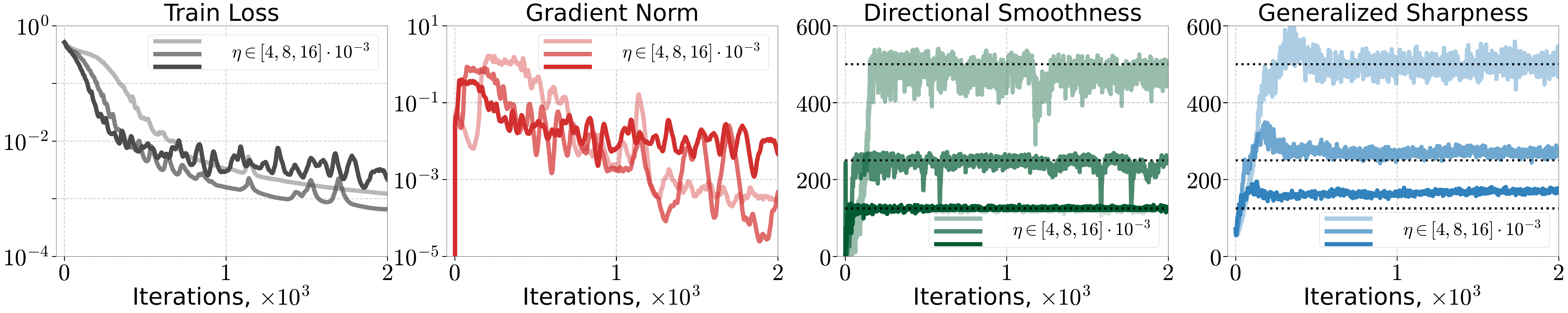} \\
        \hspace{-3mm} \includegraphics[width=1\linewidth]{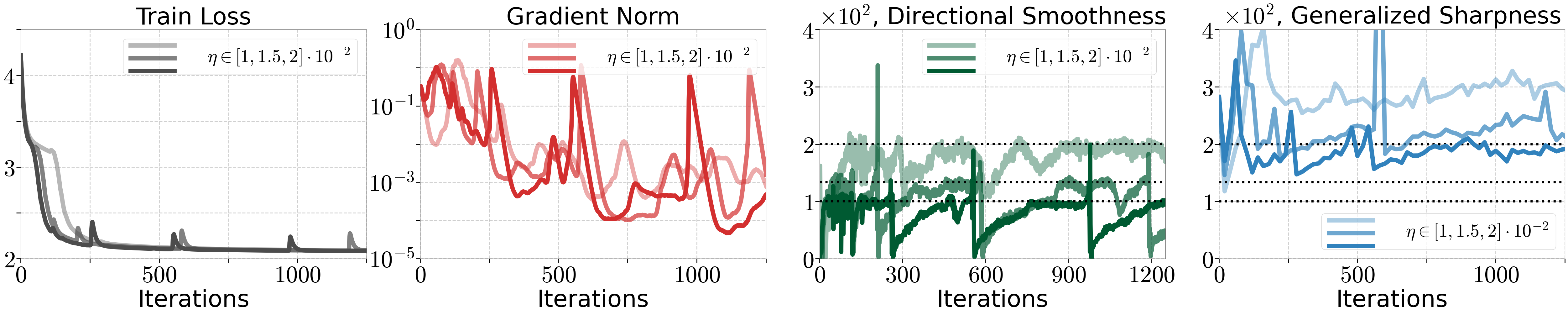}
    \end{tabular}
    \caption{\algname{Spectral GD} on {\bf Top}: MLP, CIFAR-10, {\bf bottom}: Transformer, Tiny Shakespeare. Columns show train loss, gradient norm, directional smoothness, and generalized sharpness \eqref{eq:sharpness-spectral}. Dashed lines mark the stability threshold $2/\eta$.
    }
    \label{fig:muon_all}
\end{figure*}
In this framework, one may endow each block $\mW^{\ell}$ with a matrix norm, and then define a global norm on $\mW$ by specifying an aggregation rule across layers. One particularly neat choice~\cite{bernstein2024old} is max over the spectral norms 
$$\|\mW\|_{\infty,2} \coloneqq \max_{\ell\in[L]} \|\mW^{\ell}\|_2, \; $$
where
$$\|\mW^{\ell}\|_2 \coloneqq \max_{\|\bd\|_2=1}\|\mW^{\ell}\bd\|_2.$$
Under this geometry, the dual gradient is given by the polar factor of the gradient. Concretely, the update is
\begin{align}
    \mW_{t+1}^{\ell} &= \mW_{t}^{\ell} - \eta\left(\sum_{j=1}^{L} \trace{\mS^{j}_t}\right) \mU^{\ell}_t \mV^{\ell}_t,
\end{align}
where $\mU^{\ell}_t \mS^{\ell}_t\mV^{\ell}_t = \nabla_{\mW^{\ell}}\cL(\mW_t)$ is the reduced SVD of the gradient of the $\ell$-th layer. The product $\mU_t^{\ell}\mV_t^{\ell}$ is also known as the  polar factor of the matrix $\nabla_{\mW^{\ell}}\cL(\mW_t)$, which can be computed efficiently on GPU using variants of the Newton-Schulz method \citep{jordan2024muon,higham1986computing} or the \algname{PolarExpress} \citep{amsel2025polar}. The corresponding definition of  sharpness~\eqref{eq:sharpness-norm} under this norm is
\begin{align} \label{eq:sharpness-spectral}
S^{\norm{\cdot}_{\infty,2}}(\mW)
=& \max \dotprod{\mD, \nabla^2  \cL(\mW)[\mD] } \quad \text{s.t. } \|\mD\|_{\infty,2}= 1,
\end{align}
where  the operator $\nabla^2  \cL(\mW)[\mD]$ is the directional derivative of the gradient 
$$\nabla^2  \cL(\mW)[\mD] \coloneqq \frac{d}{d\epsilon}\left.  \nabla  \cL(\mW +\epsilon \mD ) \right|_{\epsilon =0}.$$ This is exactly the operation computed by Hessian-vector-product in PyTorch \citep{paszke2019pytorch}. The solution to \eqref{eq:sharpness-spectral} cannot be computed explicitly. Therefore, we rely on the \algname{FW} algorithm to approximate it. The exact steps of \algname{FW} are derived in Lemma~\ref{lem:FW_spectral}.

Figure~\ref{fig:muon_all} presents the convergence dynamics of \algname{Spectral GD}. As in previous cases, both directional smoothness and generalized sharpness approach the stability threshold $\nicefrac{2}{\eta}$. Notably, as with the $\ell_{\infty}$ norm, the generalized sharpness gradually reaches this threshold but remains slightly above it. However, in contrast to $\ell_{\infty}$ and $\ell_{1,2}$ norms, \algname{FW} is not sensitive to the number of restarts $M$ (see Figure~\ref{fig:muon_FW_restarts}). Figure~\ref{fig:spectral_gd_gener_sharpness_vs_l2_sharpness} presents additional results on the full CIFAR-10 dataset, showing that the generalized sharpness stabilizes at the threshold $2/\eta$, whereas the $\ell_2$ sharpness remains far below throughout training. This again shows that  \algname{EoS}  only occurs with respect to the generalized sharpness, and not the standard $\ell_2$ definition of sharpness.

\begin{figure*}[!t]
    \centering
    \begin{tabular}{c}
         \includegraphics[width=1\linewidth]{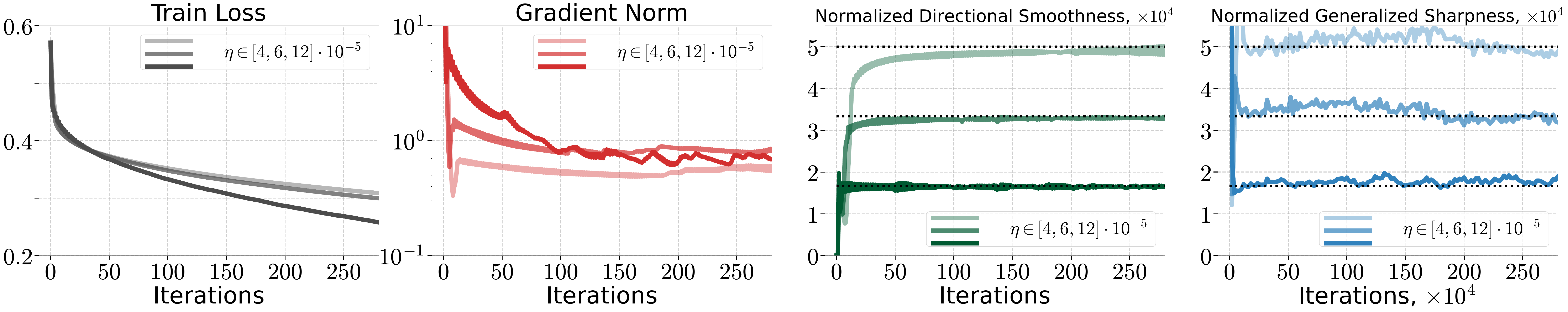} \\
         \includegraphics[width=1\linewidth]{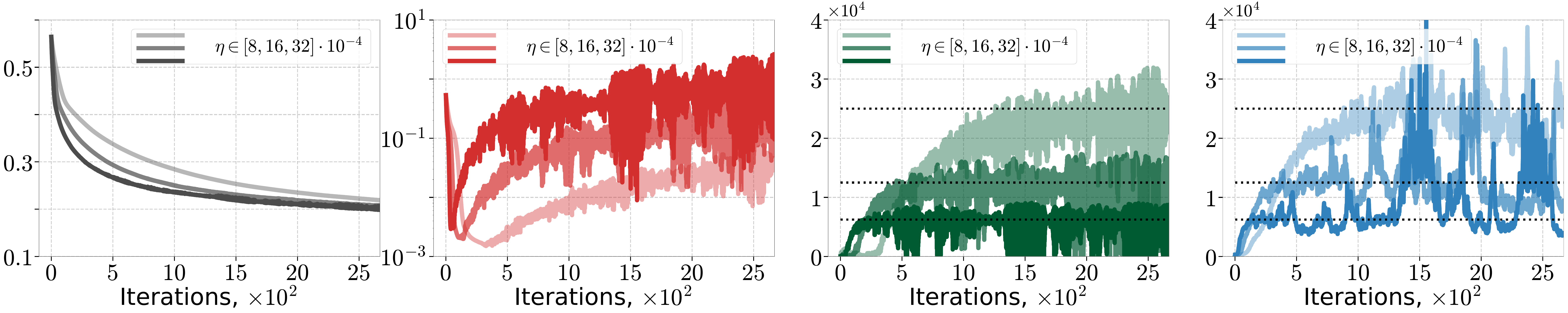}
    \end{tabular}
    \caption{Normalized non-Euclidean GD. {\bf Top}: \algname{SignGD} on CIFAR-10-5k; {\bf bottom}: normalized \algname{Spectral GD} on CIFAR-10. Columns show train loss, gradient norm, directional smoothness divided by $\|\nabla \cL(\bw_t)\|_*$, and generalized sharpness divided by $\|\nabla \cL(\bw_t)\|_*$. Dashed lines mark the stability threshold $2/\eta$.}
    \label{fig:normalized_GD_all}
\end{figure*}

\section{Normalized Non-Euclidean Gradient Descent}

In this section, we show that our  observations extend to normalized non-Euclidean \algname{GD}. In more detail, the normalized update rule \eqref{eq:normalized-gradmethod} with step-size $\eta$ can be rewritten as the unnormalized update rule \eqref{eq:gradmethod} with effective step-size $\tilde{\eta} = \frac{\eta}{\|\nabla \cL(\bw_t)\|_*}.$ Therefore, the corresponding directional smoothness $D^{\norm{\cdot}}(\bw_t, \bw_{t+1})$ and generalized sharpness of normalized non-Euclidean \algname{GD} hover at the threshold $\frac{2}{\tilde{\eta}} = \frac{2\|\nabla \cL(\bw_t)\|_*}{\eta}.$ This can also be derived by substituting one step of normalized non-Euclidean \algname{GD} into \eqref{eq:dirsmooth-opt}, giving
\begin{equation} 
    \Delta\cL_t = - \eta \left(\|\nabla \cL(\bw_t)\|_*-\frac{\eta}{2}D^{\norm{\cdot}}(\bw_{t}, \bw_{t+1})\right). \label{eq:quadratic-bound-ngd}
\end{equation}
Therefore, whenever $\|\nabla\cL(\bw_t)\|_*>0$, the loss decreases if \emph{and only if}
\begin{equation}\label{eq:D-less-than-2-eta-grad-norm}
\Delta\cL_t \le 0
\Longleftrightarrow
D^{\norm{\cdot}}(\bw_t,\bw_{t+1}) \le \frac{2\|\nabla \cL(\bw_t)\|_*}{\eta}.
\end{equation}
The derivations in Section~~\ref{sec:direct_smooth_and_sharpness} apply to normalized non-Euclidean \algname{GD}. Figure~\ref{fig:normalized_GD_all} empirically confirms the claims for \algname{SignGD} and normalized \algname{Spectral GD}
, extending our \algname{EoS} observations to more practical algorithms. We demonstrate that the directional smoothness and generalized sharpness normalized by the dual gradient norm, i.e., $\frac{D^{\norm{\cdot}}(\bw_t,\bw_{t+1})}{\|\nabla \cL(\bw_t)\|_*}$ and $\frac{S^{\norm{\cdot}}(\bw_t)}{\|\nabla\cL(\bw_t)\|_*}$ respectively, hover at the stability threshold $\nicefrac{2}{\eta}.$

Normalized $\ell_\infty$-descent can be viewed as a special case of \algname{RMSprop} with $\beta_2 = 0$. \algname{RMSprop} was shown in \citet{cohen2022adaptive} to obey a different \algname{EoS} characterization for large, typical values of $\beta_2$.  We show in Figure~\ref{fig:rmsprop} that their characterization breaks down when $\beta_2$ is small, as in the case of \algname{SignGD}.

\begin{figure*}[!t]
    \centering
    \begin{tabular}{c}
         \includegraphics[width=1\linewidth]{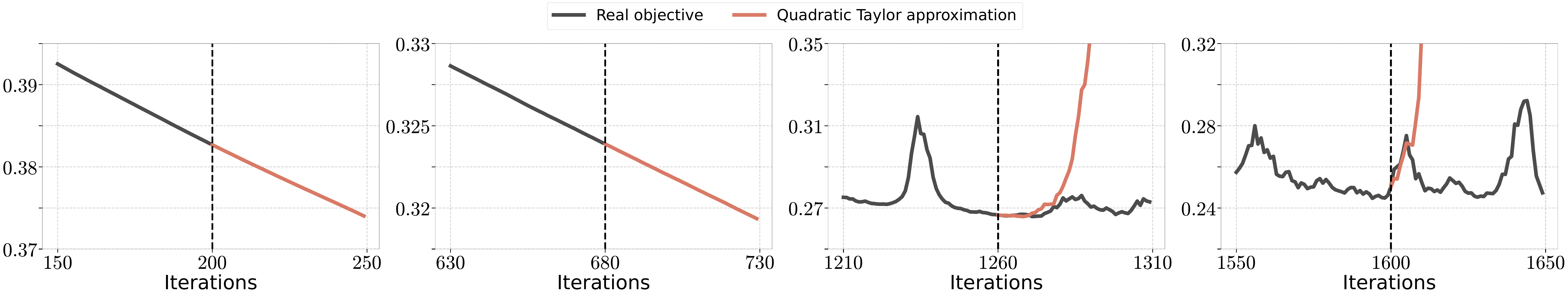} 
    \end{tabular}
    \caption{
    MSE loss for \algname{Spectral GD} on a CNN trained on CIFAR-10 with $\eta=0.002$. At four marked iterations, we switch from the true objective to the quadratic Taylor approximation at the current iterate. In the two left panels, before \algname{EoS}, the quadratic approximation closely tracks the true loss; in the two right panels, during \algname{EoS}, it quickly diverges.}
    \label{fig:quadratic_approximation_muon}
\end{figure*}

\section{Towards Understanding the Underlying Mechanism}

For vanilla \algname{GD}, the significance of $\lambda_{\max}(\nabla^2\cL(\bw_t))$ is well understood through the local quadratic Taylor approximation. On a quadratic with any eigenvalue larger than $2/\eta$, \algname{GD} oscillates with exponentially growing magnitude along the corresponding eigendirection from all but a measure-zero set of initializations. In neural networks, these local oscillations explain why the loss can temporarily increase once progressive sharpening crosses the threshold; higher-order terms, particularly cubic terms, then provide a self-stabilizing effect that can reduce sharpness and prevent divergence.

For non-Euclidean \algname{GD}, since we observe that the generalized sharpness \eqref{eq:sharpness-norm} (or at least, our estimate of it) hovers near $2/\eta$, it is natural to ask
if an analogous explanation holds.
Standard arguments from convex optimization give the following result (proof in Appendix \ref{sec:quadratics_proof}).

\begin{restatable}{theorem}{thquadraticsone}\label{th:quadratics1}
Let $\cL(\bw) \coloneqq \frac{1}{2} \bw^\top \mH\bw$ for some $\mH \succ 0$.  For some norm $\norm{\cdot}$, define the generalized sharpness $S = S^{\norm{\cdot}} := \max_{\|\bd\|\le 1} \bd^\top \mH \bd$. If we run non-Euclidean \algname{GD} (Definition~\ref{def:gradmethod}) on $\cL$ with any step-size $\eta < \nicefrac{2}{S}$, it will converge at a linear rate starting from any initial point $\bw_0$.
\end{restatable}
This theorem generalizes, to non-Euclidean norms, the fact that \algname{GD} is convergent on quadratic functions so long as the sharpness is less than $2/\eta$.
However, for the Euclidean norm, the key point is that the converse is also true: gradient descent \emph{diverges} on quadratics if the sharpness is \emph{greater} than $2/\eta$.
We now show that this property also carries over, to an extent, to the non-Euclidean setting.

\begin{restatable}{theorem}{thquadraticstwo}\label{th:quadratics2}
Let $\cL(\bw) \coloneqq \frac{1}{2} \bw^\top \mH\bw$ for some $\mH \succ 0$.  For some norm $\| \cdot \|$, define the generalized sharpness $S := \max_{\|\bd\|\le 1} \bd^\top \mH \bd$. If we run non-Euclidean \algname{GD} (Definition~\ref{def:gradmethod}) on $\cL$, there exists an initialization $\bw_0$ and a valid dual-gradient selection from which \algname{GD} diverges for any step-size $\eta > \nicefrac{2}{S}$.
\end{restatable}
The full proof is in Appendix \ref{sec:quadratics_proof}, and the crux is the following lemma, which implies that the direction $\hat \bd$ which attains the argmax in the generalized sharpness optimization problem is an invariant direction under the non-Euclidean \algname{GD} update:
\begin{restatable}{lemma}{leminvariant} \label{lem:invariant}
If $\hat{\bd} \in  \argmax_{\|\bd\| = 1}  \,\bd^{\top} \mH \bd$, then 
$(\mH \hat{\bd})_* = \hat{\bd}$.
\end{restatable}
As a result, if the iterate is initialized in $\bw_0 \in \operatorname{span}(\hat \bd)$ then the evolution of $\bw_t$ is given by:
\begin{equation}
    \bw_t = (1 - \eta S)^t \bw_0.
\end{equation}
When $\eta > 2/S \iff S > 2/\eta $, these dynamics oscillate with growing magnitude and diverge. However, we note that Th.~\ref{th:quadratics2} is less strong than what is true for Euclidean \algname{GD}, as Euclidean \algname{GD} diverges from all but a zero-measure set of initializations, whereas Th.~\ref{th:quadratics2} only establishes divergence when the initialization is on a particular line.

Empirically, we can assess whether non-Euclidean \algname{GD} is indeed divergent on the quadratic Taylor approximation when operating on the edge of stability.
In Figure~\ref{fig:quadratic_approximation_muon}, for points during training both before and after entering \algname{EoS}, we switch from running non-Euclidean \algname{GD} on the real objective to running non-Euclidean \algname{GD} on the quadratic Taylor approximation (similar to Appendix E from \citet{cohen2021edgeofstability}).
We observe that \algname{GD} is stable before reaching \algname{EoS}, but divergent afterwards.
This supports the idea that the significance of the generalized sharpness hovering around $2/\eta$ is related to the dynamics becoming divergent on the local quadratic Taylor approximation.

This explanation is still incomplete: our theory proves divergence only for a specific initialization, whereas empirically the quadratic approximation appears to diverge much more generically. Bridging this gap would be an interesting question for future work.

It is worth highlighting an additional point of difference between the Euclidean and non-Euclidean cases.
For Euclidean \algname{GD}, the directional smoothness only starts to grow from $\approx 0$ to $2/\eta$ \emph{after} the sharpness crosses $2/\eta$.
By contrast, for non-Euclidean \algname{GD} under some norms (in particular, $\ell_\infty$ and $\norm{\cdot}_{2\rightarrow2}$), we observe that the directional smoothness starts to climb towards $2/\eta$ \emph{before} the generalized sharpness has reached $2/\eta$ (Appendix~\ref{sec:pre-eos-oscillatory}). During this period, we find that the iterates oscillate in weight space, but the dynamics are not yet divergent on the quadratic Taylor approximation. This suggests an intermediate regime between stability and \algname{EoS} regimes, which does not occur for Euclidean \algname{GD}.
Understanding this behavior would be an interesting question for future work.

\section{Conclusion and Future Work}

We proposed a geometry-aware view of the edge of stability for non-Euclidean \algname{GD}. The key object is directional smoothness, which yields an exact loss-change identity and motivates a generalized sharpness \(S^{\norm{\cdot}}\) adapted to the norm defining the update geometry. This generalized sharpness recovers the standard maximum Hessian eigenvalue for Euclidean \algname{GD} and the preconditioned-Hessian sharpness for preconditioned \algname{GD}, while extending naturally to geometries such as \(\ell_\infty\)-descent, \algname{Block CD}, \algname{Spectral GD}, and their normalized variants. 
Across MLP, CNN, and Transformer experiments, the proposed sharpness exhibits progressive sharpening and then hovers near (or just above) the corresponding stability threshold. In contrast  the standard \(\ell_2\) sharpness can fail to capture the observed \algname{EoS} behavior for non-Euclidean \algname{GD}.

Several questions remain open. 
First, our quadratic analysis explains stability below \(2/\eta\) and gives a divergence construction above \(2/\eta\), but it does not yet establish generic divergence for arbitrary initialization in general non-Euclidean norms. 
Second, non-Euclidean methods can exhibit a pre-\algname{EoS} oscillatory regime in which directional smoothness increases before generalized sharpness reaches the threshold; understanding this intermediate regime is an important direction for future work. 
Finally, extending the framework beyond full-batch non-Euclidean \algname{GD} to stochastic, momentum-based, and adaptive optimizers would clarify how geometry-aware \algname{EoS} diagnostics should be used in practical large-scale training.

\section*{Acknowledgement}

Rustem Islamov acknowledges the financial support of the Swiss National Foundation, SNSF grant No 207392.

\bibliography{references}
\bibliographystyle{plainnat}

\newpage
\appendix
\counterwithin{figure}{section}
\counterwithin{table}{section}

\vbox{
  {\hrule height 2pt \vskip 0.15in \vskip -\parskip}
  \centering
  {\LARGE\bf Appendix\par}
  {\vskip 0.2in \vskip -\parskip \hrule height 0.5pt \vskip 0.09in}
}

\newcommand\invisiblepart[1]{%
  \refstepcounter{part}%
  \addcontentsline{toc}{part}{\protect\numberline{\thepart}#1}%
}

\invisiblepart{Appendix}
\setcounter{tocdepth}{2}
\localtableofcontents

\appendixtrue

\section{Discussion on Frank-Wolfe Algorithm}\label{sec:frank_wolfe}

Solving \eqref{eq:sharpness-norm} reduces to the quadratic maximization problem
\begin{equation}\label{eq:max_quadratic}
\max_{\|\bu\|\le 1}\bu^\top \mH \bu,
\end{equation}
for an arbitrary norm $\norm{\cdot}$ and symmetric matrix $\mH$. Even in the convex case where $\mH$ is positive definite, problem \eqref{eq:max_quadratic} is NP-hard \citep{burer2009nonconvex} and is recognized as a fundamental challenge in global optimization \citep{horst2000introduction}. Consequently, without exploiting additional structure, global optimality guarantees cannot be expected from generic first-order methods. Instead, one can provide stationarity-type guarantees or approximation bounds via relaxations \citep{burer2009nonconvex}.

The Frank-Wolfe (\algname{FW}) algorithm is a projection-free method that relies on a linear oracle. For maximization problems such as \eqref{eq:max_quadratic}, each step maximizes the linearization of $\bu^\top\mH\bu$ over the norm ball, equivalently minimizing the linearization of $-\bu^\top\mH\bu$. For $L$-smooth functions over convex domains, which includes \eqref{eq:max_quadratic}, the \algname{FW} algorithm provides convergence to approximate stationary points, measured through the Frank-Wolfe gap
\[
    \cG(\bu) \coloneqq \max_{\|\bw\|\le 1}\langle\bw-\bu,-\mH\bu \rangle,
\]
where the last term comes with a minus sign since we minimize $-\bu^\top\mH\bu$. Specifically, \algname{FW} identifies an iterate $\bu_K$ satisfying $\cG(\bu_K) \le \varepsilon$ in $\cO(\nicefrac{1}{\varepsilon^2})$ iterations, i.e., at rate $\cO(1/\sqrt{K})$ \citep{lacoste2016convergence}. While this guarantee does not imply global optimality for \eqref{eq:max_quadratic}, it provides a principled and certifiable stopping criterion. If the optimal value of \eqref{eq:max_quadratic} is positive, a global maximizer can be chosen on the boundary of the unit ball. Therefore, in the experiments, we add a projection step. In our experiments, this final projection consistently improved the reported estimate.

As an alternative, consider the projected power iteration
\[
\bu_{k+1} = \Pi_{\norm{\cdot}}(\mH\bu_k).
\]
For the Euclidean norm, this reduces to the classical Power method, which converges to the normalized leading eigenvector provided the initialization has a nonzero component along it \citep{golub2013matrix}. For general norms, however, no global convergence guarantees are known: the projected iterates can stall or even cycle--for example, when they approach generalized eigenvectors, namely unit vectors $\bv$ that are fixed points of the linear minimization oracle, $\bv=\argmin{\norm{\bw}=1}\langle\bw-\bv,-\mH\bv \rangle$. Empirically, we found that \algname{FW} provides a good estimate of \eqref{eq:sharpness-norm} when a sufficient number of restarts is used. 

\section{An oscillatory regime before \algname{EoS}}
\label{sec:pre-eos-oscillatory}

In this appendix, we briefly elaborate on an oscillatory regime that occurs for some optimizers (including $\ell_\infty$-descent and \algname{Spectral GD}) \emph{before} the algorithm reaches \algname{EoS}.  This stands in contrast to Euclidean \algname{GD}, which generally does not oscillate before the sharpness reaches $2/\eta$ \citep{cohen2025centralflow}.  

In Figure~\ref{fig:pre-eos}, we train a network using $\ell_\infty$-descent.  Initially, the generalized sharpness is less than $2/\eta$, the directional smoothness is $\approx 0$, and the network’s predictions are not oscillating.  Then, around step 300, even though the generalized sharpness is less than $2/\eta$, the directional smoothness starts to rise and the network’s predictions start to oscillate, which are indications that the iterates are oscillating in weight space.  Finally, around step 450, the generalized sharpness and directional smoothness reach $2/\eta$ and the algorithm reaches \algname{EoS}. The network’s predictions oscillate wildly.

The existence of the pre-\algname{EoS} oscillatory regime is interesting, since no such regime exists for Euclidean \algname{GD}.

\begin{figure}[t]
    \centering
    \begin{tabular}{c}
    \hspace{-3mm}\includegraphics[width=0.9\linewidth]{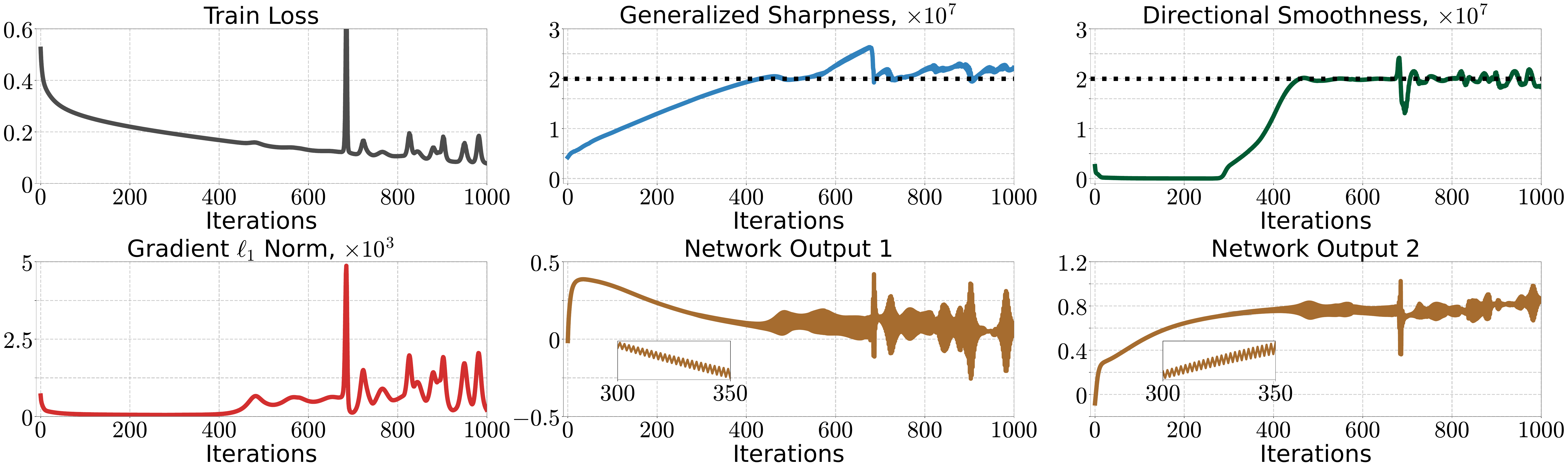} 
    \end{tabular}
    \caption{ \textbf{An oscillatory regime before \algname{EoS}}.  We train a network using $\ell_\infty$-descent. From steps $\sim$300$ - $450, the generalized sharpness is less than $2/\eta$ (so the algorithm is not yet at \algname{EoS}), but the directional smoothness has already started to climb from $\approx 0$ towards $2/\eta$, and the network's predictions have already started to oscillate.  This would not occur for Euclidean \algname{GD}.  This network is a fully connected network trained on a subset of CIFAR-10 using MSE loss and $\eta = 1 \times 10^{-7}$. }
    \label{fig:pre-eos}
\end{figure}

In Figure~\ref{fig:pre-eos-run-quadratic}, we further explore this phenomenon.  At three points during training, we switch from running $\ell_\infty$-descent on the real objective to running it on the quadratic Taylor approximation.  We show the evolution of the network output under the resulting trajectory.  Initially (left), the network output does not oscillate, indicating that the iterates are not oscillating in weight space.  On the other hand, once the dynamics are in the pre-\algname{EoS} oscillatory regime (middle), the network output oscillates but does not diverge.  Finally, once the dynamics are at \algname{EoS} (right), the network output diverges.

An interesting avenue for future work would be to understand why non-Euclidean \algname{GD} starts to oscillate when it does.

\begin{figure}[t]
    \centering
    \begin{tabular}{c}
    \hspace{-3mm}\includegraphics[width=0.9\linewidth]{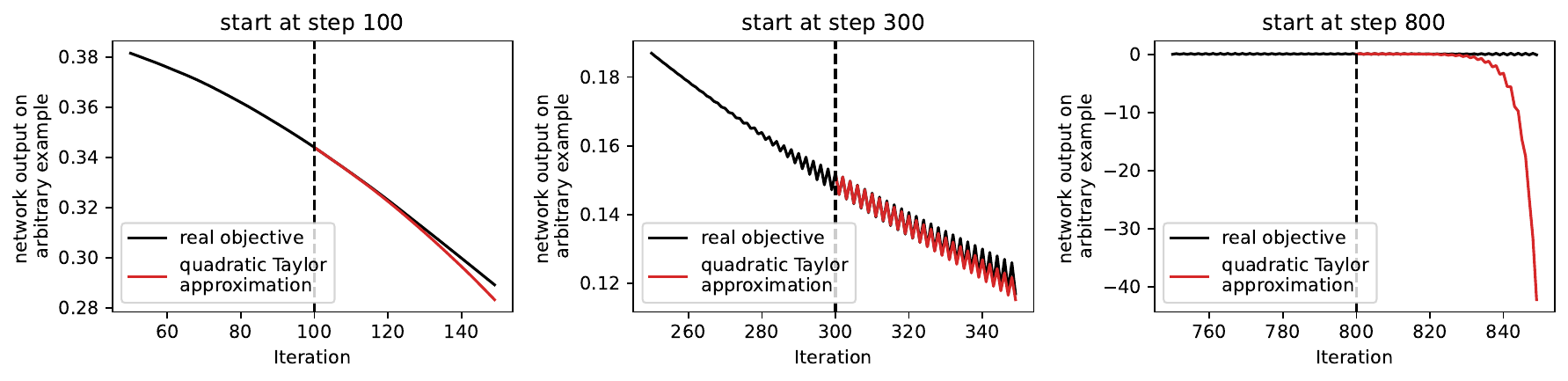} 
    \end{tabular}
    \caption{ \textbf{In the pre-\algname{EoS} oscillatory regime, training on the quadratic Taylor approximation oscillates without diverging.}  While training the network from Figure~\ref{fig:pre-eos}, we switch from training on the real objective to training on the quadratic Taylor approximation at three points during training: at step 100 (while the optimizer is stable and non-oscillatory), at step 300 (while the optimizer is in the pre-\algname{EoS} oscillatory regime), and at step 800 (when the network is at \algname{EoS}).  For these trajectories, we plot the network's output on an arbitrary test example.  In the first case, this output evolves smoothly; in the third case, it diverges; and, interestingly, in the second case, it oscillates with sustained magnitude and without diverging.  }
    \label{fig:pre-eos-run-quadratic}
\end{figure}

\section{The gap between the generalized sharpness and $2/\eta$}\label{sec:gap_generalized_sharpness_and_eta}
\begin{figure}[b]
    \centering
    \begin{tabular}{c}
    \hspace{-3mm}\includegraphics[width=0.9\linewidth]{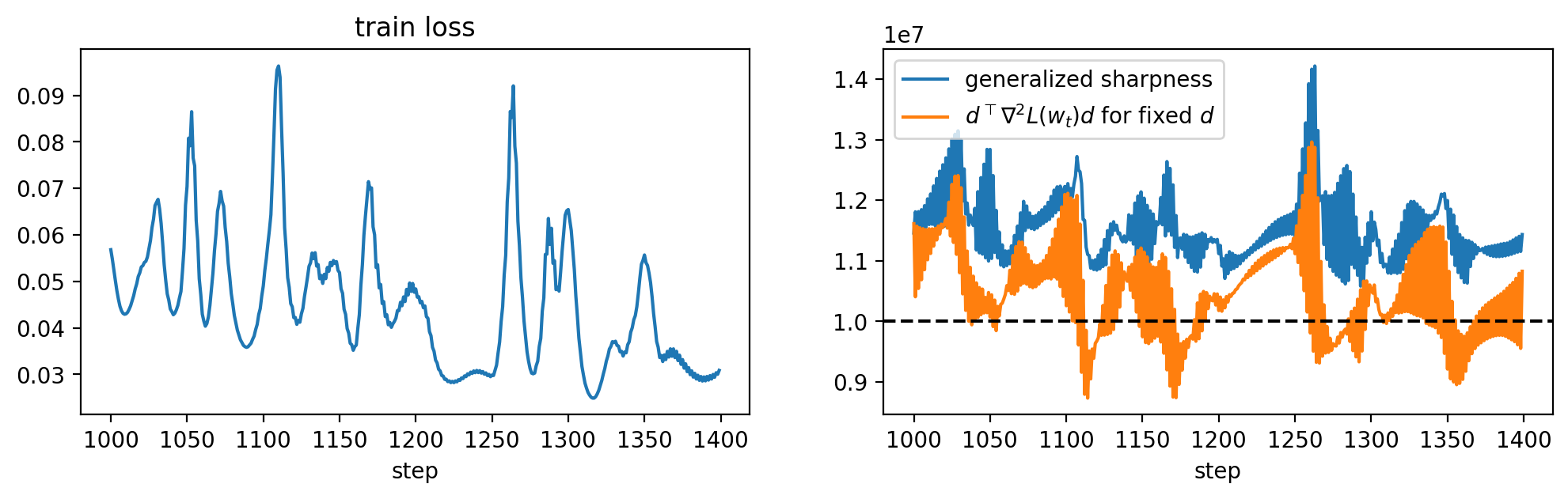} 
    \end{tabular}
    \caption{For a stretch of training, we plot both the (estimated) generalized sharpness $\max_{\| \bd \| \le 1} \, \bd^\top \nabla^2 \cL(\bw_t) \bd$ (blue), as well as the quadratic form ${\bd_*}^\top \nabla^2 \cL(\bw_t) {\bd_*}$ where $\bd_* \in \operatorname{argmax}_{\| \bd \| \le 1} \, \bd^\top \nabla^2 \cL(\bw_{t_0}) \bd$  is the maximizing direction at step $t_0 = 1000$.  While the first quantity is consistently larger than $2/\eta$, the second is much closer to $2/\eta$.  This is a fully-connected network trained on a subset of CIFAR-10 using MSE loss and $\ell_\infty$-descent with $\eta = 2\cdot 10^{-7}$.}
    \label{fig:gen-sharpness-vs-fixed-d}
\end{figure}
Prior studies of Euclidean \algname{GD} at \algname{EoS} have observed that there is often a gap between the sharpness and $2/\eta$; for example, in Figure 1 of \citet{cohen2021edgeofstability}, the sharpness can be seen to sometimes exceed the critical threshold of $2/\eta$ by 150\%.
Similar effects can be observed in plots throughout this paper for the generalized sharpness during non-Euclidean \algname{GD}.
We now review the prevailing explanation for this phenomenon for Euclidean \algname{GD}, and suggest that a similar mechanism is at play for non-Euclidean \algname{GD}.

For Euclidean \algname{GD}, \citet{cohen2025centralflow} argue that when multiple Hessian eigenvalues are near $2/\eta$, \algname{GD} should be conceived of as oscillating within the subspace spanned by the corresponding eigenvectors.
The \algname{EoS} phenomenon is that for every direction $\bd$ in this subspace, the local time-average of the directional curvature $\bd^\top \nabla^2 \cL(\bw) \bd$ is approximately equal to $2/\eta$.
Concretely, if at some iteration $t$, one computes the top Hessian eigenvector $\bd$, and then monitors the quantity $\bd^\top \nabla^2 \cL(\bw_{t+j}) \bd$ for the next $j=1, \ldots, m$ iterations, then the local time-average of this quantity $\frac{1}{m} \sum_{j=1}^m \bd^\top \nabla^2  \cL(\bw_{t+j}) \bd$ is predicted to be approximately $2/\eta$.
By contrast, if we compute the top Hessian eigenvalue anew at every iteration $\{\lambda_{\max}(\nabla^2 \cL(\bw_t))\}$, then due to the chaotic oscillatory dynamics, we get back a different vector within this subspace at every step, and because the largest Hessian eigenvector is the direction with the largest curvature, there is an upward bias.

For an analogy, consider the random $d$-dimensional matrix
\begin{align*}
    \mH := \mU [\frac{2}{\eta}\, \mI_k + \varepsilon \operatorname{diag}(\bz)] \mU^\top, \quad \bz \sim \cN(0, \mI_k),
\end{align*}
where $\mU \in \mathbb{R}^{d \times k}$ has orthogonal columns and $\epsilon > 0$ is a small number.
Here, $\mH$ is an analogy to the Hessian, the columns of $\mU$ are the $k\ge2$ unstable Hessian eigenvectors, and the random noise $\bz$ is an analogy to the chaotic oscillatory dynamics.
The nonzero eigenvalues of $H$ are exactly $\frac{2}{\eta} + \epsilon \, \bz$, and so the largest eigenvalue $\lambda_{\max}(\mH)$ is precisely $\frac{2}{\eta} + \epsilon \, \max_{1 \le i\le k} \, z_i$.
It can be shown that $\mathbb{E}[ \max_{1 \le i\le k} \, z_i] > 0$ provided that $k \ge 2$, and thus we have $\mathbb{E}[\lambda_{\max}(\mH)] > \frac{2}{\eta}$.
On the other hand, for any fixed vector $\bv \in \mathrm{Range}(\mU)$, we have that $\frac{\mathbb{E}[\bv^\top \mH \bv]}{\|\bv\|^2} = \frac{2}{\eta} $.

Generalizing this argument to the case of non-Euclidean \algname{GD} is nontrivial, as in the non-Euclidean case we do not yet know if there is an analogous concept to multiple eigenvalues being at the edge of stability.
Nevertheless, in Figure~\ref{fig:gen-sharpness-vs-fixed-d}, we empirically show that while the generalized sharpness \eqref{eq:sharpness-norm} hovers strictly above $2/\eta$, if we fix a timestep $t_0$ and compute the maximizer $\bd$ of the generalized sharpness problem \eqref{eq:sharpness-norm} at this timestep, then the quadratic form $\bd^\top \nabla^2 \cL(\bw_{t_0+j}) \, \bd$ computed over the next $j=1, \ldots, m$ steps is much closer to $2/\eta$.

\section{Training Details}

Our implementation is based on open source code from \citet{cohen2021edgeofstability} together with publicly available datasets. In all our experiments, we use algorithms with full-batch gradient, i.e., we run them in the deterministic setting. The datasets and step-sizes $\eta$ used in the experiments are specified in the figures. If not specified, we use the Frank-Wolfe algorithm with $M=5$ restarts and $K=50$ iterations, and \algname{PolarExpress} with 5 steps. 

In the training of CNN and MLP models, we use MSE loss, while in the training of the Transformer model, we use CE loss.

\section{Additional Experimental Results with $\ell_{\infty}$ Descent}

\subsection{Convergence When Training CNN Model}

In \Cref{fig:signgd_all_appendix}, we present additional results when training a CNN model on the CIFAR-10-5k dataset using $\ell_\infty$-descent. We observe a similar behavior that the generalized sharpness approximated by the \algname{FW} algorithm hovers at the stability threshold $2/\eta.$

\begin{figure}[!t]
    \centering
    \begin{tabular}{c}
         \hspace{-3mm}\includegraphics[width=1\linewidth]{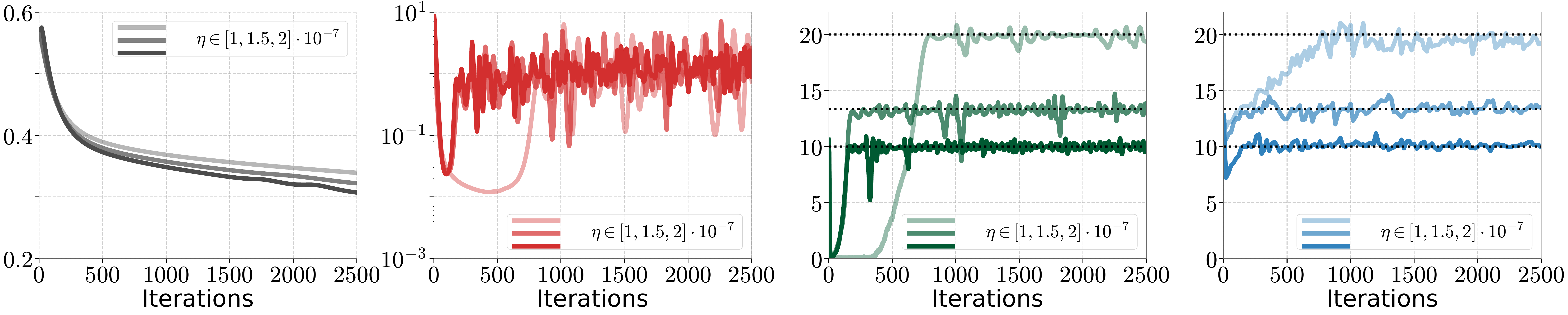}
    \end{tabular}
    
    \caption{($\ell_{\infty}$-descent) Train loss, gradient norm, directional smoothness, and generalized sharpness \eqref{eq:max-inner-norm} during training CNN on CIFAR10-5k with $\ell_{\infty}$-descent. Horizontal dashed lines correspond to the value $\nicefrac{2}{\eta}$. Gradient norm and train loss curves are smoothed using an exponential smoothing with $\alpha=0.1.$ We use \algname{FW} with $K=50$ and $M=5$ to approximate \eqref{eq:max-inner-norm}.}
    \label{fig:signgd_all_appendix}
\end{figure}

\begin{figure}
    \centering
    \begin{tabular}{c}
        \hspace{-3mm}\includegraphics[width=1\linewidth]{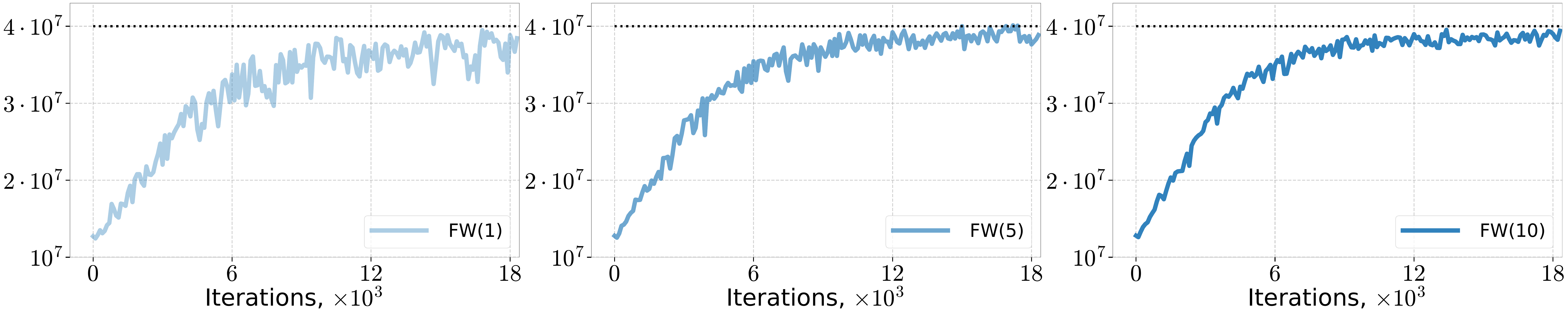} 
    \end{tabular}

    \caption{The approximation of the generalized sharpness of $\ell_{\infty}$-descent by the Frank-Wolfe algorithm varying the number of initialization points in $\{1,5,10\}$ for the Frank-Wolfe algorithm. Here, FW($k$) denotes $k$ restarts of the Frank-Wolfe algorithm, with varying initialization points.}
    \label{fig:signgd_FW_restarts}
\end{figure}

\subsection{Sensitivity of Frank-Wolfe Algorithm in Estimating the generalized sharpness for Sign Gradient Descent}

In this section, we study the sensitivity of the Frank-Wolfe algorithm in estimating the generalized sharpness of non-Euclidean gradient descent methods. Our experiments are conducted on a CNN with two convolutional layers, followed by a linear layer, trained on the CIFAR10-5k dataset \citep{krizhevsky2009learning}. We run $\ell_{\infty}$-descent, and approximate the generalized sharpness by Frank-Wolfe with $50$ iterations, using $\{1, 5, 10\}$ initialization points drawn from a standard normal distribution, and take the maximum over restarts as the generalized sharpness estimate.

In Figure~\ref{fig:signgd_FW_restarts}, we show that the Frank-Wolfe estimate of the generalized sharpness is sensitive to the number of restarts. With a single random initialization, the algorithm generally underestimates the value. Increasing the number of restarts to $15$ yields a much more stable estimate that closely aligns with the true value almost everywhere.

\subsection{Results on ResNet20 and VGG11}\label{sec:l_inf_gen_sharpness_vs_l2_sharpness}
In this section, we provide additional empirical results on larger models, such as ResNet20 \citep{he2016deep} and VGG11 \citep{simonyan2014very}, trained on the CIFAR-10 dataset with $\ell_{\infty}$-descent and MSE loss. From the results in Figure~\ref{fig:signgd_gener_sharpness_vs_l2_sharpness}, we observe that both directional smoothness and generalized sharpness hover at the stability threshold $2/\eta$. In contrast, a standard notion of sharpness, i.e., $\lambda_{\max}(\nabla^2\cL(\bw_t))$ defined in the Euclidean norm, lies significantly below the threshold (brown line in the right subfigure). Note that for the ResNet20 model, the generalized sharpness stabilizes slightly above the threshold due to several unstable directions as explained in Section~\ref{sec:gap_generalized_sharpness_and_eta}.

\begin{figure}[!t]
    \centering
    \begin{tabular}{c}
         \includegraphics[width=1\linewidth]{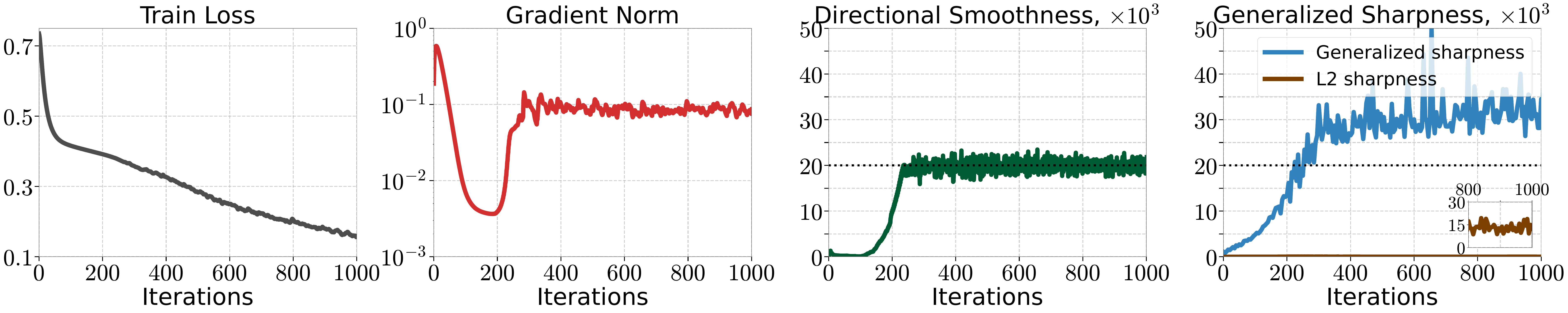} \\
         \includegraphics[width=1\linewidth]{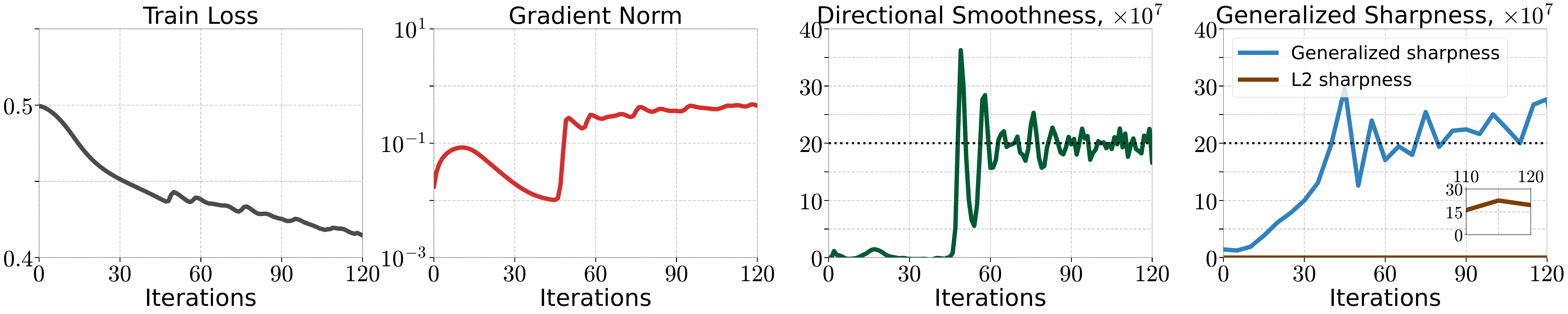}
    \end{tabular}
    \caption{($\ell_{\infty}$-descent) Train loss, gradient norm, directional smoothness, generalized sharpness \eqref{eq:max-inner-norm}, and L2 sharpness ($\lambda_{\max}(\nabla^2\cL(\bw_t))$) during training ResNet20 (top, $\eta=10^{-4}$) and VGG11 (bottom, $\eta=10^{-7}$) on CIFAR-10 with $\ell_{\infty}$-descent. Horizontal dashed lines correspond to the value $\nicefrac{2}{\eta}$.} 
    \label{fig:signgd_gener_sharpness_vs_l2_sharpness}
\end{figure}

\section{Additional Experimental Results with Block Gradient Descent}

\subsection{Sensitivity of Frank-Wolfe Algorithm in Estimating the Generalized Sharpness for Block Gradient Descent}

In this section, we study the sensitivity of the Frank-Wolfe algorithm in estimating the generalized sharpness of non-Euclidean gradient descent methods. Our experiments are conducted on a CNN with four convolutional layers, followed by a linear layer, trained on the CIFAR10-5k dataset \citep{krizhevsky2009learning}. Now we evaluate Block GD, where the generalized sharpness has a closed-form expression \eqref{eq:max_block_eigen}. We run Frank-Wolfe for $50$ iterations, using $\{1, 7, 15\}$ initialization points drawn from a standard normal distribution, and take the maximum over restarts as the generalized sharpness estimate. The Frank-Wolfe procedure is applied every $100$ iterations of \algname{Block CD}.

In Figure~\ref{fig:blockgd_true_vs_fw}, we show that the Frank-Wolfe estimate of the maximum block-wise Hessian eigenvalue is sensitive to the number of restarts. With a single random initialization, the algorithm provides a good approximation at a few iterations but generally underestimates the value. Increasing the number of restarts to $15$ yields a much more stable estimate that closely aligns with the true value almost everywhere.

\begin{figure}
    \centering
    \begin{tabular}{c}
        \hspace{-3mm}\includegraphics[width=1\linewidth]{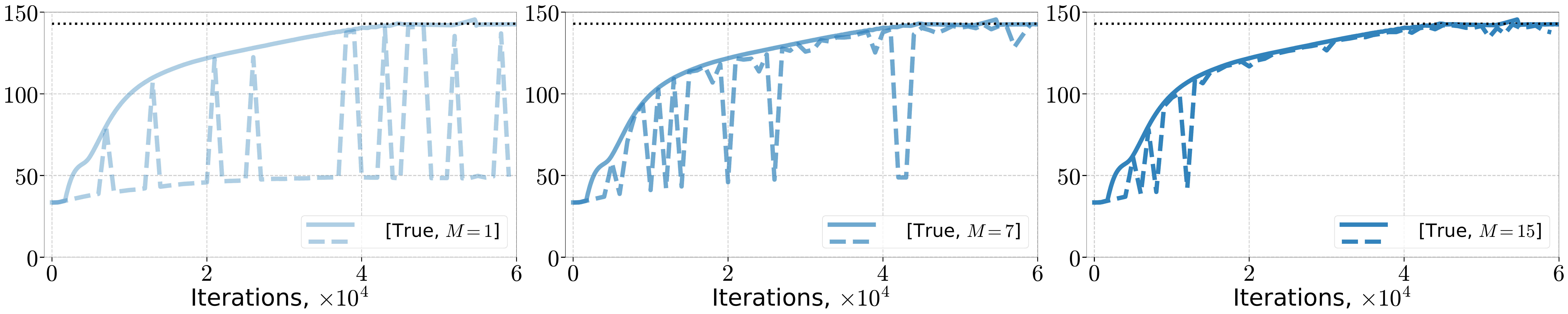} 
    \end{tabular}

    \caption{The maximum block-wise Hessian eigenvalue (solid line), which is the generalized sharpness of \algname{Block CD}, and its approximation by the Frank-Wolfe algorithm varying the number of initialization points in $\{1,7,15\}$ for the Frank-Wolfe algorithm. Here, $M$ is the number of restarts of the Frank-Wolfe algorithm, varying the initialization point.}
    \label{fig:blockgd_true_vs_fw}
\end{figure}

\section{Additional Experimental Results with Spectral Gradient Descent}

\subsection{Convergence When Training CNN Model}

In this section, we present the results when training CNN model on CIFAR-10 dataset with \algname{Spectral GD}; see Figure~\ref{fig:muon_all_app}. The results support our theoretical observations.

\begin{figure}[!t]
    \centering
    \begin{tabular}{c}
        \hspace{-3mm} \includegraphics[width=1\linewidth]{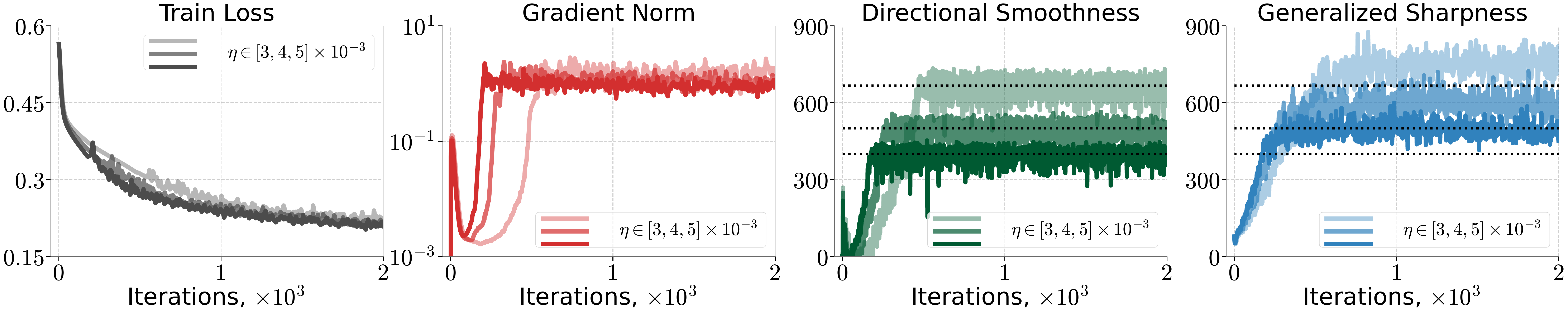}
    \end{tabular}
    \caption{(\algname{Spectral GD}) Train loss, gradient norm,  directional smoothness, and generalized sharpness \eqref{eq:sharpness-spectral} during training CNN model on CIFAR-10 dataset with the \algname{Spectral GD}. Horizontal dashed lines correspond to the value $\nicefrac{2}{\eta}$.}
    \label{fig:muon_all_app}
\end{figure}

\subsection{Sensitivity of Frank-Wolfe Algorithm in Estimating the Generalized Sharpness for Spectral Gradient Descent}

Next, we switch to the \algname{Spectral GD} to train CNN model on the full CIFAR-10 dataset. We perform a similar procedure to the one done in the previous section. We fix the number of Polar Express steps in both \algname{Spectral GD} and Frank-Wolfe to $5$ and vary the number of initialization points for Frank-Wolfe in $\{1, 5, 10\}.$ Each run of Frank-Wolfe has $50$ iterations.

In Figure~\ref{fig:muon_FW_restarts}, we observe that \algname{Spectral GD} is less sensitive to the number of initialization points for Frank-Wolfe than Block GD. Therefore, it is not necessary to do restarts for Frank-Wolfe when it is used to measure the generalized sharpness of the \algname{Spectral GD} algorithm. 

\begin{figure}
    \centering
    \begin{tabular}{c}
        \hspace{-3mm}\includegraphics[width=1\linewidth]{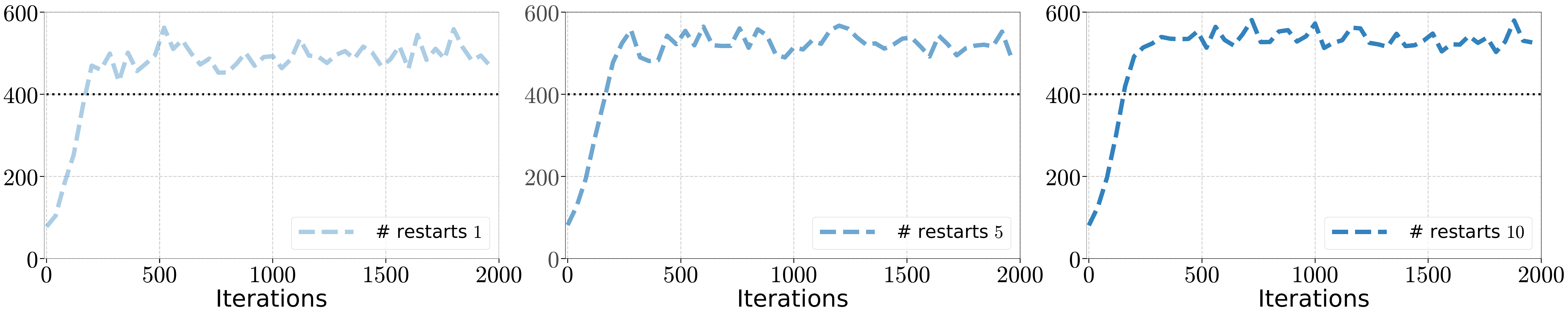} 
    \end{tabular}

    \caption{The approximation of the generalized sharpness by the Frank-Wolfe algorithm for \algname{Spectral GD} varying the number of initialization points in $\{1,5,10\}$ for the Frank-Wolfe algorithm.}
    \label{fig:muon_FW_restarts}
\end{figure}

\subsection{Sensitivity of Spectral Gradient Descent to the Number of Polar Express Steps}

We investigate how the number of Polar Express steps affects the generalized sharpness estimation of \algname{Spectral GD}. To this end, we fix the number of Polar Express steps in \algname{Spectral GD} and vary the number of steps in the Frank-Wolfe algorithm across $\{5, 10, 15\}$, and vice versa. All experiments are conducted using a CNN with four convolutional layers, trained on the full CIFAR-10 dataset.

As shown in Figure~\ref{fig:polar_express_steps}, we do not observe any significant differences across the different configurations. This indicates that $5$ steps of the Polar Express algorithm are sufficient to obtain an accurate and stable estimate of \algname{Spectral GD}’s generalized sharpness.

\begin{figure}
    \centering
    \begin{tabular}{c}
        \hspace{-3mm}\includegraphics[width=0.6\linewidth]{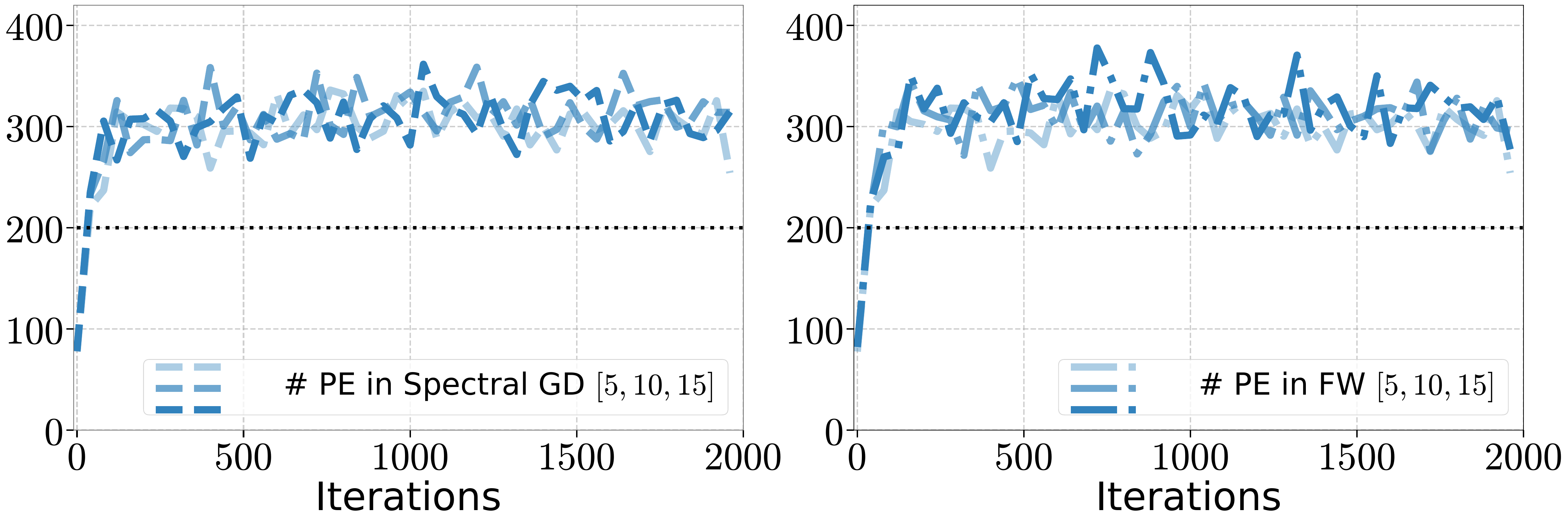} 
    \end{tabular}

    \caption{The sensitivity of the generalized sharpness estimation of \algname{Spectral GD} to the number of Polar Express steps in \algname{Spectral GD} (left) and in Frank-Wolfe (right). Here $\#$ PE means the number of Polar Express steps in \algname{Spectral GD} or Frank-Wolfe algorithm respectively.}
    \label{fig:polar_express_steps}
\end{figure}

\subsection{Quadratic Taylor Approximation of the Real Objective}

In this section, we provide additional results when training CNN model Figure~\ref{fig:muon_all} with \algname{Spectral GD}. At some iteration (indicated in Figure~\ref{fig:appendix_quadratic_approximation_muon}), we switch from running the algorithm on the real objective to its quadratic approximation at that point using exactly the same hyperparameters. We observe that during progressive sharpening phase (iterations 200 and 400), the dynamics on the quadratic loss approximate well those on the real objective. In opposite, the dynamics on the quadratic model when \algname{Spectral GD} is at \algname{EoS} already, the quadratic loss quickly diverges.

\begin{figure}[!t]
    \centering
    \begin{tabular}{c}
         \hspace{-3mm}\includegraphics[width=1\linewidth]{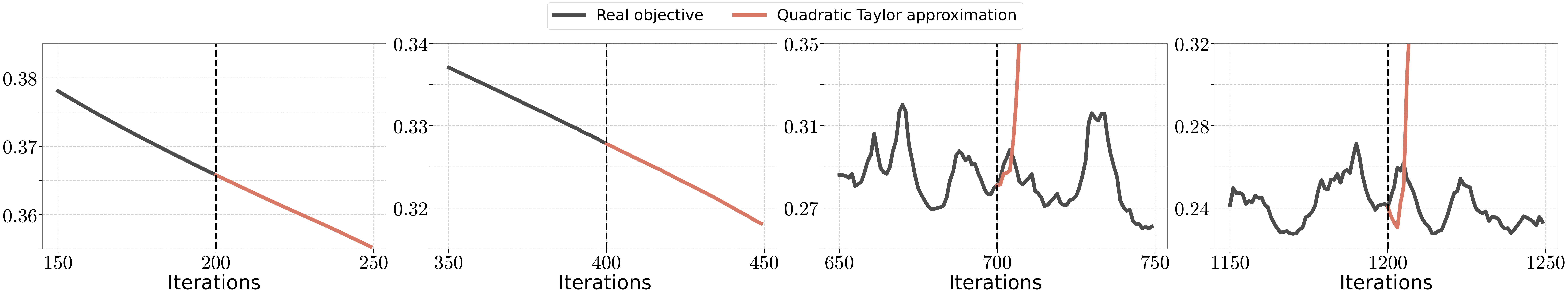}  \\ 

         \hspace{-3mm}\includegraphics[width=1\linewidth]{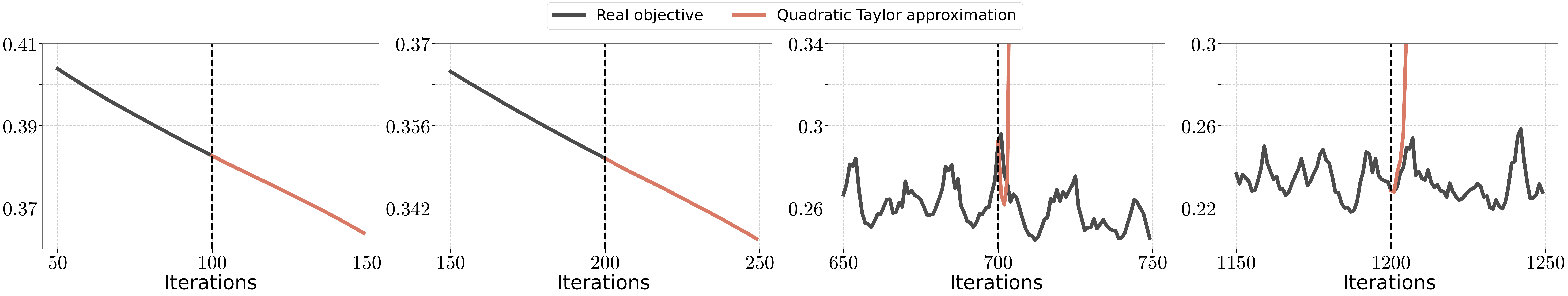} 
    \end{tabular}
    
    \caption{MSE loss (top row $\eta=0.003$, bottom row $\eta=0.004$). At 4 different iterations during the training of the CNN from Figure~\ref{fig:muon_all} (marked by the vertical dotted black lines), we switch from running \algname{Spectral GD} on the real neural training objective (for which the train loss is plotted in gray) to running \algname{Spectral GD} on the quadratic Taylor approximation around the current iterate (for which the
train loss is plotted in orange). Two left figures are timesteps before \algname{Spectral GD} has entered \algname{EoS}; observe that the orange line (Taylor approximation) closely tracks the blue line (real objective). Two right figures are timesteps during the \algname{EoS}; observe that the orange line quickly diverges, whereas the blue line does not.}
    \label{fig:appendix_quadratic_approximation_muon}
\end{figure}

\subsection{Results on ResNet20 and VGG11}\label{sec:spectral_gd_gen_sharpness_vs_l2_sharpness}
In this section, we provide additional empirical results on larger models, including ResNet20 \citep{he2016deep} and VGG11 \citep{simonyan2014very}, trained on the CIFAR-10 dataset using \algname{Spectral GD} with MSE loss. As shown in Figure~\ref{fig:spectral_gd_gener_sharpness_vs_l2_sharpness}, both the directional smoothness and the generalized sharpness remain close to the stability threshold $2/\eta$. In contrast, the standard notion of sharpness--namely $\lambda_{\max}(\nabla^2\cL(\bw_t))$ computed in the Euclidean norm--stays well below this threshold (brown curve in the right panel). For the ResNet20 model, the generalized sharpness stabilizes slightly above $2/\eta$, which can be attributed to the presence of several unstable directions, as discussed in Section~\ref{sec:gap_generalized_sharpness_and_eta}.

\begin{figure}[!t]
    \centering
    \begin{tabular}{c}
         \includegraphics[width=1\linewidth]{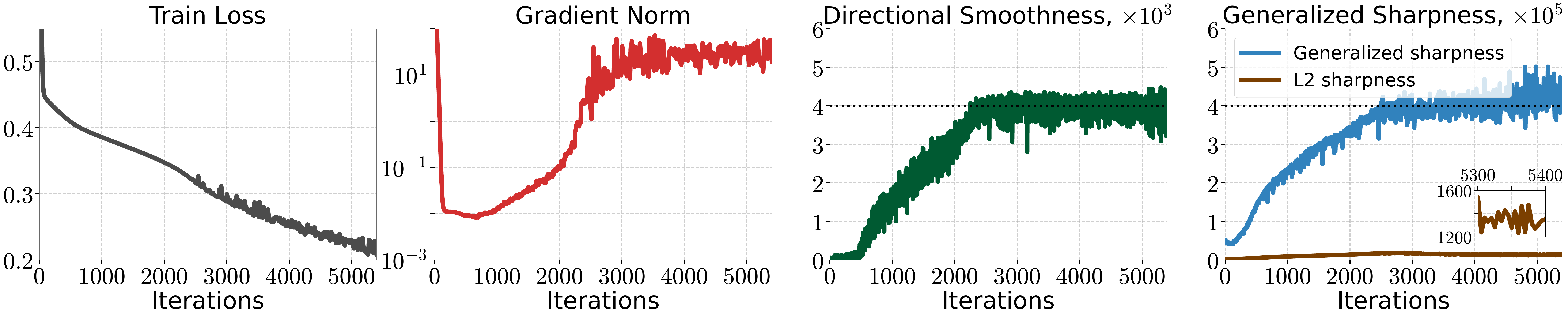} \\
         \includegraphics[width=1\linewidth]{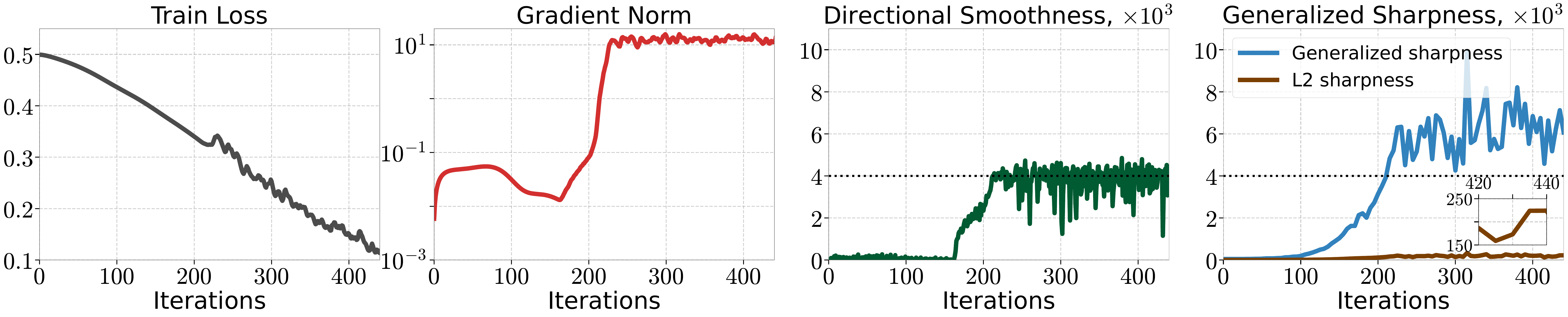}
    \end{tabular}
    \caption{(\algname{Spectral GD}) Train loss, gradient norm, directional smoothness, generalized sharpness \eqref{eq:sharpness-spectral}, and L2 sharpness ($\lambda_{\max}(\nabla^2\cL(\bw_t))$) during training ResNet20 (top, $\eta=5\cdot 10^{-5}$) and VGG11 (bottom, $\eta=5\cdot 10^{-4}$) on CIFAR-10 with \algname{Spectral GD}. Horizontal dashed lines correspond to the value $\nicefrac{2}{\eta}$.} 
    \label{fig:spectral_gd_gener_sharpness_vs_l2_sharpness}
\end{figure}

\section{$\ell_{\infty}$-descent and \algname{RMSprop}}
In this section, we report results for the \algname{RMSprop} algorithm when training an MLP on the CIFAR10-5k subset with MSE loss. Although \algname{SignGD} can be viewed as a limiting case of \algname{RMSprop} as $\beta_2\to 0$, the adaptive \algname{EoS} (\algname{AEoS}) condition of \citet{cohen2022adaptive} is valid only when $\beta_2$ is large (i.e., close to 1 in practical settings) and breaks down as $\beta_2$ becomes small. For small $\beta_2$, the largest eigenvalue of the preconditioned Hessian $\lambda_{\max}(\mP^{-1}_t\nabla^2\cL(\bw_t))$ does not stabilize around $2/\eta$; instead, it often exceeds this value by a substantial margin. The underlying issue is that as $\beta_2\to 0$, the algorithm no longer resembles preconditioned gradient descent with a slowly-changing preconditioner, which is the approximation that inspires the \algname{AEoS} condition.

Our results in Figure~\ref{fig:rmsprop} support this observation. We plot the top four eigenvalues of the preconditioned Hessian for \algname{RMSprop}, showing that they stabilize around the threshold $2/\eta$ only when $\beta_2$ is large, while for small $\beta_2$ the behavior deviates significantly.

\begin{figure}[!h]
    \centering
    \begin{tabular}{c}
         \hspace{-3mm}\includegraphics[width=1\linewidth]{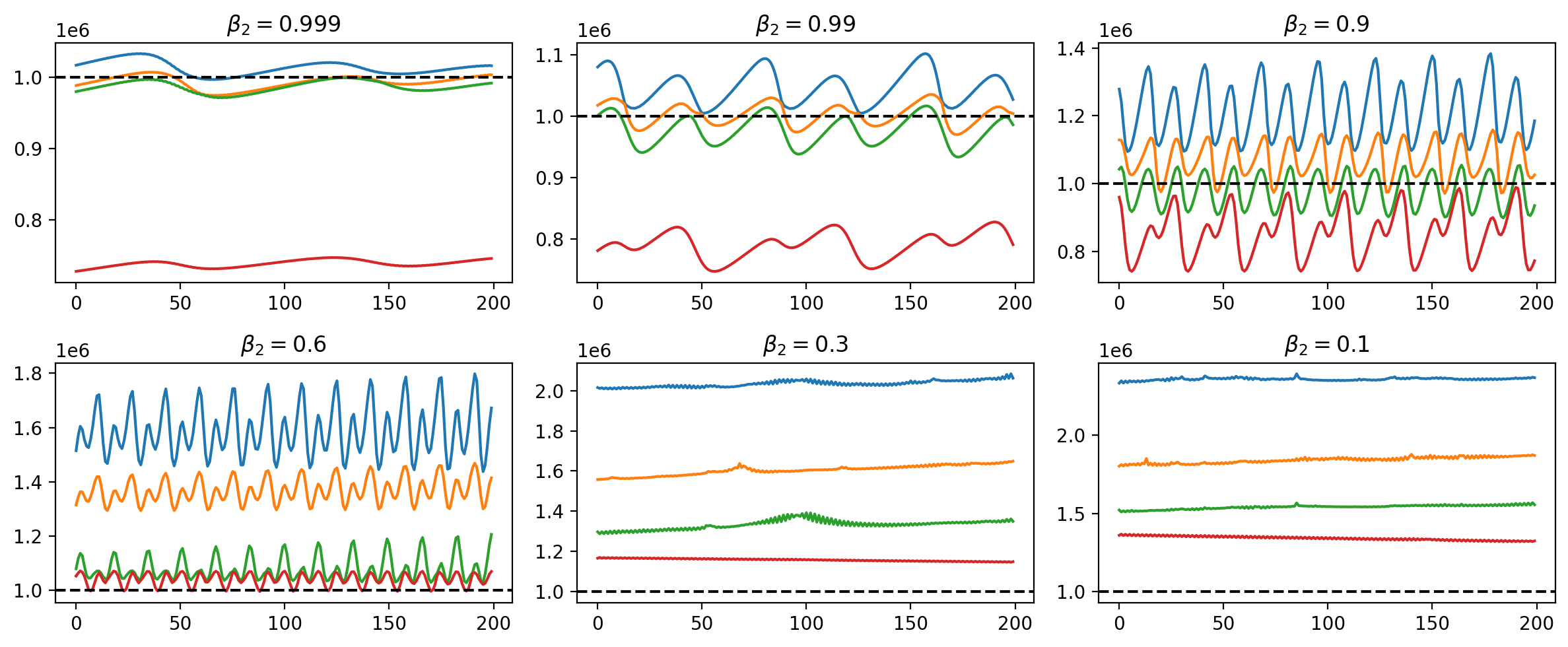}  
    \end{tabular}
    
    \caption{Sharpness of \algname{RMSprop} when training MLP model on a subset of CIFAR-10 dataset, varying $\beta_2$ hyperparameter. Here, colored lines correspond to the evolution of the top-4 largest eigenvalues of the preconditioned Hessian, while the dashed line is $2/\eta$ threshold. We observe that \algname{RMSprop} reaches \algname{AEoS} only for realistic (close to 1) values of $\beta_2$, while for small $\beta_2$ the preconditioned sharpness is not at $2/\eta$, but significantly higher. }
    \label{fig:rmsprop}
\end{figure}

\section{Useful Lemmas}

\subsection{Missing Proofs for the Spectral Block Norm $\ell_{\infty,2}$}

First, we derive the step of \algname{Spectral GD}.
\begin{lemma}\label{lem:dualpdatesum}
Let $\|\mX^{\ell}\|_{\cW_{\ell}}$ be the norm of the $\ell$-th layer and $\|\mX\|^2 = \sum_{\ell=1}^L\|\mX^{\ell}\|_{\cW_\ell}^2$. The solution to 
\begin{align}
    \DW_* = \argmin{\DW} \trace{\DW^\top\mG}+ \frac{1}{2\eta} \| \DW\|^2.
\end{align}
is given by
\begin{equation}
    \DW^{\ell}_* = \eta \cdot \|\mG^{\ell}\|_{\cW_{\ell}}^* \cdot \argmin{\|\mX\|_{\cW_{\ell}} =1}\trace{\mX^\top\mG^{\ell}}
\end{equation}
where $\norm{\cdot}_{\cW_{\ell}}^*$ denotes the dual norm of $\|\cdot \|_{\cW_{\ell}}.$
\end{lemma}
\begin{proof}
First, note that this problem is separable over each layer since
\begin{align*}
     \trace{\DW^\top\mG}+ \frac{1}{2\eta} \| \mW\|^2 
    &= \sum_{\ell=1}^L \left(\trace{(\DW^{\ell})^\top\mG^{\ell}}+ \frac{1}{2\eta} \| \DW^{\ell}\|_{\cW_{\ell}}^2\right). 
\end{align*}
Thus, we can solve over each layer separately. Changing coordinates with $\DW^{\ell} = c \mX$ where $\|\mX\|_{\cW_{\ell}} =1$ and $c \geq 0$ we have that 
\begin{align*}
  \min_{\DW^{\ell}}  \trace{(\DW^{\ell})^\top\mG^{\ell}}+ \frac{1}{2\eta} \| \DW^{\ell}\|_{\cW_{\ell}}^2 &=  \min_{c \geq 0}c \min_{\|\mX\|_{\cW_{\ell}} =1} \trace{\mX^\top\mG^{\ell}}+ \frac{1}{2\eta}c^2 \\
  &= \min_{c \geq 0}-c \|\mG^{\ell}\|_{\cW_{\ell}}^*+ \frac{1}{2\eta}c^2. 
\end{align*}
Here, we use the fact that $\argmin{\mX}{\rm tr}(\mX^\top \mG^{\ell}) = -(\mG^{\ell})^*$ is the dual matrix of $\mG^{\ell}$. Finally solving in $c\geq 0$ gives $c = \eta \cdot \|\mG_{\ell}\|_{\cW_{\ell}}^*.$
 
\end{proof}

If we use the infinity norm over layers instead of the Euclidean one, we get the following result. 
\begin{lemma} \label{lem:dualpdatemax}
The solution to 
    \begin{align}
    \DW_* = \argmin{\DW} \trace{\DW^\top\mG_t}+ \frac{1}{2\eta}\max_{\ell \in [L]}  \|\DW^{\ell}\|_{\cW_{\ell}}^2.
\end{align}
is given by
\begin{equation}
    \DW^{\ell}_* =  \eta \gamma \cdot \argmin{\|\mX\|_{\cW_{\ell}} =1}\trace{\mX^\top\mG^{\ell}_t}
\end{equation}
where $\gamma \coloneqq \sum_{\ell=1}^L \|\mG^{\ell}_t\|_{\cW_{\ell}}^*$ and $\norm{\cdot}_{\cW_{\ell}}^*$ denotes the dual norm of $\|\cdot \|_{\cW_{\ell}}.$
\end{lemma}
\begin{remark}
    If $\|\cdot \|_{\cW_{\ell}} = \|\cdot \|_2$ for all $\ell\in[L],$ then $\DW^{\ell} = \eta\gamma\mU^{\ell}_t\mV^{\ell}_t$ where $\mG^{\ell}_t = \mU^{\ell}_t\mS^{\ell}_t\mV^{\ell}_t$ is the reduced SVD decomposition. Moreover, $\gamma=\sum_{\ell=1}^L\|\mG_t^{\ell}\|_*$ is the sum of nuclear norms over the layers. See the proof in \citep{bernstein2024old}. 
\end{remark}

\begin{proof}
    The problem that we want to solve is 
    \[
    \min\limits_{\DW}\sum_{\ell=1}^L\trace{(\DW^{\ell})^\top \mG^\ell_t} + \frac{1}{2\eta}\max_{\ell\in[L]}\|\DW^\ell\|^2_{\cW_{\ell}};
    \]
    Let $\cS \coloneqq \{\DW\mid \|\DW^{\ell}\|_{\cW_{\ell}} \le t~\forall \ell\in[L]\}$. We can rewrite this problem as 
    \begin{align*}
    &\min_{t\ge 0}\min_{\DW\in\cS}\left[\sum_{\ell=1}^{L}\trace{(\DW^{\ell})^\top\mG^{\ell}_t} + \frac{1}{2\eta}\|\DW^{\ell}\|_{\cW_{\ell}}^2\right]
    = \min_{t\ge 0}\min_{\DW\in\cS}\left[\sum_{\ell=1}^{L}\trace{(\DW^{\ell})^\top\mG^{\ell}_t} + \frac{t^2}{2\eta}\right]\\
    =\;& \min_{t\ge 0}\left[\sum_{\ell=1}^{L}\min_{\|\DW^\ell\|_{\cW_{\ell}}\le t}\trace{(\DW^{\ell})^\top\mG^{\ell}_t} + \frac{t^2}{2\eta}\right]
    = \min_{t\ge 0}\left[\sum_{\ell=1}^{L}-t\max_{\|\DW^\ell\|_{\cW_{\ell}}\le 1}\trace{(\DW^{\ell})^\top\mG^{\ell}_t} + \frac{t^2}{2\eta}\right]\\
    =\;& \min_{t\ge 0}\left[\sum_{\ell=1}^{L}-t\|\mG^{\ell}_t\|_{\cW_\ell}^* + \frac{t^2}{2\eta}\right].
    \end{align*}
    Now it is a quadratic problem in $t$. The minimizer $t_*$ is given by 
    \[
    t_* \coloneqq \eta\sum_{\ell=1}^L\|\mG^{\ell}_t\|_{\cW_\ell}^*.
    \]
    Therefore, the final solution is given by 
    \[
    \DW^{\ell} = \eta \left(\sum_{\ell=1}^L \|\mG^{\ell}_t\|^*_{\cW_{\ell}}\right)\argmin{\|\mX\|_{\cW_{\ell}}\le 1}\trace{\mX^\top \mG^{\ell}_t}.
    \]

\end{proof}

\begin{lemma}\label{lem:FW_spectral}
    Let $\norm{\cdot}$ be the spectral block norm $\norm{\cdot}_{2\rightarrow2}$. Then the iterates of the \algname{FW} to approximate \eqref{eq:sharpness-spectral} are given by
    \[
        \mU^{\ell}_k\mV^{\ell}_k = {\rm polar}(\nabla_{\mW^{\ell}}F(\mD_k)), \quad \mD_{k+1}^{\ell} = (1-\gamma_k)\mD_k + \gamma_k\mU^{\ell}_k\mV^{\ell}_k,
    \]
    where ${\rm polar}(\cdot)$ is the polar decomposition of a matrix, $\gamma_k=\frac{2}{2+k}$
\end{lemma}
\begin{proof}
We consider the Frank-Wolfe method for finding an approximate solution. For shortness, let $\mH \coloneqq \nabla^2  \cL(\mW_t)$, and  note that the objective $F(\mD) \coloneqq  \dotprod{\mD, \mH[\mD]}$ is a quadratic form, whose gradient is given by
\[ \nabla F(\mD) \; = \; 2   \mH[\mD].\]
To compute a step of the Frank-Wolfe method, we need to solve
\begin{align*}
\argmax{\mD} \dotprod{\nabla F(\mD_k), \mD }
    & \qquad \mbox{subject to } \| \mD^{\ell}\|_{2}  \leq 1, \quad \mbox{for }\ell =1,\ldots, L.
\end{align*}
Clearly, this problem is separable over layers and is thus equivalent to solving \citep{bernstein2024old}
\begin{align*}
\mU^{\ell}_k \mV^{\ell}_k =  \argmax{\mD^{\ell}} \dotprod{\nabla_{\mW^\ell} F(\mD_k), \mD^{\ell} }
    & \qquad \mbox{subject to } \| \mD^{\ell}\|_{2}  \leq 1,
\end{align*}
where $\nabla_{\mW^\ell} F(\mD_k)$ is the directional derivative of the gradient of the $\ell$-th layer given by
\[ 
    \nabla_{\mW^\ell} F(\mD_k) =  \frac{d}{d\epsilon}\left.  \nabla_{\mW^{\ell}}  \cL(\mD^{1}_k, \ldots,\mD^{\ell}_k  +\epsilon \mD^{\ell}, \ldots, \mD^L_k ) \right|_{\epsilon =0} 
\]
and where $\mU^{\ell}_k \mS^{\ell}_k \mV^{\ell}_k  = \nabla_{\mW^\ell} F(\mD_k)$. The matrix $\mU^{\ell}_k \mV^{\ell}_k $ is also known as the polar factor of  $\nabla_{\mW^\ell} F(\mD_k)$.  The resulting Frank-Wolfe method is thus given by

\begin{equation*}
 \mU^{\ell}_k \mV^{\ell}_k = \mbox{polar}(\nabla_{\mW^\ell} F(\mD_k)), \quad
    \mD_{k+1}^{\ell} = (1-\gamma_k) \mD_{k}^{\ell} + \gamma_k \mU^{\ell}_k \mV^{\ell}_k,
\end{equation*}
where $\gamma_k = \frac{2}{k+2}.$
\end{proof}

\subsection{Missing Proofs for the Block $\ell_{1,2}$ Norm}

\begin{lemma}\label{lem:block_gd}
        The solution to the problem 
        \[
        \Delta\bw_* = \argmin{\bw}\langle\Delta\bw,\bg_t \rangle + \frac{1}{2\eta}\|\Delta\bw\|^2_{1,2}
        \]
        can be written as 
        \[
        \Delta\bw^\ell_* = 
        \begin{cases}
            0 & \text{if } \bg_t = 0,\\
            0 & \text{if } \bg_t \neq 0 \text{ and } \ell \notin J,\\
            -\frac{\eta}{|J|}\bg^\ell_t & \ell\in J,
        \end{cases}
        \]
        where $J \coloneqq \{\ell\in[L]\mid \|\bg_t^\ell\|_2 = \max_{j\in[L]} \|\bg_t^j\|_2\}.$
    \end{lemma}

    \begin{remark}
        In the case when $J$ is a singleton, we obtain \algname{Block CD} 
        \[
        \bw_{t+1}^{\ell} =
        \begin{cases}
          \bw_t^{\ell} - \eta\bg^\ell_t & \mbox{ if }\ell =   \ell_{\max}, \\
           \bw_t^{\ell} & \mbox{ otherwise, }
        \end{cases}
        \]
        where $\ell_{\max} = \argmax{\ell\in[L]}\|\bg_t^\ell\|_2.$
    \end{remark}

    \begin{remark}
        In the case when $L=d$, we obtain vanilla coordinate descent (\algname{CD})
        \[
        \bw_{t+1}^j = \begin{cases} \bw_t^{j_{\max}} - \eta\bg^{j_{\max}}_t & \mbox{ if } j=j_{\max}\\
        \bw_t^j & \mbox{ otherwise,}
        \end{cases}
        \]
        where $j_{\max} = \argmax{j\in[d]}|\bg^j_t|$.
    \end{remark}
    \begin{proof}
        We need to find a solution to the problem 
        \[
        \min_{\Delta\bw} \langle\Delta\bw, \bg_t\rangle + \frac{1}{2\eta}\left(\sum_{\ell=1}^L\|\Delta\bw^\ell\|_2\right)^2 =  \min_{\Delta\bw} \sum_{\ell=1}^L\langle\Delta\bw^\ell, \bg_t^\ell \rangle + \frac{1}{2\eta}\left(\sum_{\ell=1}^L\|\Delta\bw^\ell\|_2\right)^2
        \]
        Let $\Delta\bw_*$ be the solution to the problem. Therefore, 
        \begin{align}\label{eq:necessary_cond}
        0 &\in \bg_t+ \frac{1}{\eta}\left(\sum_{\ell=1}^L\|\Delta\bw^\ell_*\|_2\right)\partial\left(\sum_{\ell=1}^L\|\Delta\bw_*^\ell\|_2\right)\notag\\
        &= \bg_t+ \frac{1}{\eta}\left(\sum_{\ell=1}^L\|\Delta\bw^\ell_*\|_2\right)(\partial \|\Delta\bw^1_*\|_2^\top, \ldots, \partial\|\Delta\bw^L_*\|_2^\top )^\top.
        \end{align}
        Let $\chi = \sum_{\ell=1}^L\|\Delta\bw_*^\ell\|_2$. Note that 
        \[
            \partial\|\bx\| = \begin{cases}
            \frac{\bx}{\|\bx\|_2} & \text{if } \bx \neq 0,\\
            \{\by \mid \|\by\|_2 \le 1\} & \text{otherwise}
        \end{cases}.
        \]
        Therefore, we should satisfy the following $L$ equalities
        \begin{align}\label{eq:sfqmfkq}
        -\bg^\ell_t = \frac{\chi}{\eta}\partial\|\Delta\bw^\ell_*\|_2, \quad \text{and} \quad \|\bg^\ell_t\|_2 = \frac{\chi}{\eta}\left\|\partial\|\Delta\bw^\ell_*\|_2\right\| \le \frac{\chi}{\eta}.
        \end{align}
        This implies that each block of $\bg_t$ has a norm at most $\chi/\eta,$ and whenever some block $\ell$ satisfies $\partial\|\Delta\bw^\ell_*\|_2 = \frac{\Delta\bw^\ell_*}{\|\Delta\bw^\ell_*\|_2}$, then the corresponding block $\|\bg^\ell_t\|_2=\frac{\chi}{\eta}.$
        
        If $\|\bg^\ell_t\|_2 = 0$ for all $\ell\in[L],$ i.e., $\bg_t= 0$, then for all $\Delta\bw^\ell_* = 0.$ 
        
        Now let us assume that there is at least one block $\ell\in[L]$ such that $\|\bg^\ell_t\|_2 \neq 0.$ Let $J\coloneqq \{\ell \in [L] \mid \|\bg^\ell_t\|_2 = \max_{j \in[L]} \|\bg^j_t\|_2\} \neq \emptyset$. Then, for all blocks $\ell\in J$ we have $\|\bg^\ell_t\|_2 = \frac{\chi}{\eta}.$ Indeed, if it is not the case, i.e., if for all $\ell\in[L]$ we have $\|\bg^\ell_t\|_2 < \frac{\chi}{\eta},$ then $\Delta\bw^* = 0$ and we obtain a contradiction to \eqref{eq:necessary_cond} since $\bg_t\neq 0$.

        We summarize that for any block $\ell\notin J$ such that $\|\bg^\ell\|_2 < \frac{\chi}{\eta}$ we obtain $\Delta\bw_*^\ell=0.$ In the opposite case for $\ell\in J$, we have that 
        \[\|\bg_t^\ell\|_2 = \max_{j\in[L]}\|\bg^j_t\|_2 = \frac{\chi}{\eta} \Rightarrow \chi = \sum_{\ell \in J} \|\Delta\bw_*^\ell\|_2 = |J|\max_{\ell\in J}\|\Delta\bw_*^\ell\| =  \eta\max_{\ell\in[L]}\|\bg^\ell_t\|,
        \]
        and from \eqref{eq:sfqmfkq} we obtain $\Delta\bw^\ell_* = -\frac{\eta\max_{j \in [L]}\|\bg_t^j\|_2}{|J|} \frac{\bg^\ell_t}{\|\bg^\ell_t\|_2} = -\frac{\eta}{|J|}\bg^{\ell}_t$ for $\ell \in J.$ This concludes the proof.
    \end{proof}

\begin{lemma}\label{lem:sharpness_l12_norm}
    Let $\norm{\cdot}$ be the block $\ell_{1,2}$ norm. Assume that the Hessian $\nabla^2\cL(\bw_t)$ is positive semi-definite. Then the generalized sharpness \eqref{eq:sharpness-12norm} is given by
    \[
        S^{\norm{\cdot}_{1,2}}(\bw_t) = \max_{\ell\in[L]}\lambda_{\max}(\nabla^2_{\bw^{\ell}}(\bw_t)).
    \]
\end{lemma}
\begin{proof}
If $\mH = \nabla^2\cL(\bw_t)$ is positive semidefinite, then the function $f(\bd) = \dotprod{\bd, \mH\bd}$ is convex. Our goal is to find the maximum of this quadratic convex function over a $\ell_{1,2}$-norm unit ball. It attains the maximum at the border, i.e., $\|\bd\|_{1,2}=1.$ Any point  $\by$ at the border of the $\ell_{1,2}$ unit norm can be expressed as 
\[
    \by = (\alpha_1\bd^1,\ldots,\alpha_L\bd^L) \quad \text{where} \quad \|\bd^\ell\|_2 = 1~\forall \ell\in[L]\quad \text{and} \quad \sum_{\ell=1}^L\alpha_\ell = 1.
\]
Let $\by_1 = (\bd^1, 0, \ldots, 0), \by_2 = (0, \bd^2, \ldots, 0), \ldots, \by_L = (0, 0, \ldots, \bd^L),$ $\|\bd^{\ell}\|_2 = 1$ for all $\ell\in[L].$ Then $\by = \sum_{\ell=1}^L \alpha_\ell\by_\ell.$ Since $f$ is convex, then $f(\by) \le \sum_{\ell=1}^L \alpha_\ell f(\by_\ell) \le \max_{\ell\in[L]}f(\by_\ell).$ Therefore, our problem reduces to 
\begin{align}\label{eq:max_block_eigen}
&\max_{\ell\in[L]}\max_{\|\bd^\ell\|_2=1} \dotprod{\bd^\ell, \nabla^2_{\bw^\ell}\cL(\bw_t)\bd^{\ell}} = \max_{\ell\in[L]}\lambda_{\max}(\nabla^2_{\bw^\ell}\cL(\bw_t)),
\end{align}
where $\nabla^2_{\bw^\ell}\cL(\bw_t)$ is the $\ell$-th diagonal block of the Hessian. In the special case of $L=d$, we have the sharpness measure

\[ \max_{\bd} \frac{\bd^\top \nabla^2 \cL(\bw_t) \bd}{\norm{\bd}_{1}^2} =\max_{j}|\nabla^2 \cL(\bw_t)_{jj}|.\]
\end{proof}

\begin{lemma}\label{lem:FW_l12_norm}
    Let $\norm{\cdot}$ be the block $\ell_{1,2}$ norm. Then the iterates of the \algname{FW} to approximate \eqref{eq:sharpness-12norm} are given by
    \[
    \bv_k = \frac{(\nabla^2\cL(\bw_t)\bd_k)_{\ell}}{\|(\nabla^2\cL(\bw_t)\bd_k)_{\ell}\|_2}, \quad \bd_{k+1}= (1-\gamma_k)\bd_k + \gamma_k \bv_k,
    \]
    where $(\nabla^2\cL(\bw_t)\bd_k)_{\ell}$ is the $\ell$-th block of the vector $\nabla^2\cL(\bw_t)\bd_k$, and $\gamma_k = \frac{2}{2+k}.$
\end{lemma}
\begin{proof}
    We consider the Frank-Wolfe method for finding an approximate solution. For shortness, let $\mH \coloneqq \nabla^2\cL(\bw_t),$ and note that the objective $F(\bd) \coloneqq\bd^\top\mH\bd$ is a quadratic form, whose gradient is given by $\nabla F(\bd) = 2\mH\bd.$ To compute a step of the Frank-Wolfe method, we need to solve 
    \[
    \argmax{\bd}\langle\nabla F(\bd_k), \bd\rangle \quad \mbox{subject to } \|\bd\|_{1,2} \le 1.
    \]
    The solution to this is given by the dual norm and the dual gradient
    \[
    \max_{\|\bd\|_{1,2}\le 1} \langle\nabla F(\bd_k), \bd\rangle = \|\nabla F(\bd_k)\|_{\infty,2} = \max_{\ell\in[L]}\|\nabla_{\bd^{\ell}}F(\bd_k)\|_2.
    \]
    This is true, since 
    \begin{align}
    \langle\nabla F(\bd_k), \bd \rangle &= \sum_{\ell=1}^L \langle\nabla_{\bd^\ell} F(\bd_k), \bd^{\ell} \rangle \le \sum_{\ell=1}^L\|\nabla_{\bd^\ell} F(\bd_k)\|_2\cdot\| \bd^{\ell}\|_2\nonumber\\
    &\le \max_{\ell\in[L]}\|\nabla_{\bd^\ell} F(\bd_k)\|_2\cdot\sum_{\ell=1}^L\|\bd^{\ell}\|_2 = \max_{\ell\in[L]}\|\nabla_{\bd^\ell} F(\bd_k)\|_2.
    \end{align}
    The maximizer is obtained by concentrating all mass on any group $\ell \in \{\ell : \|\nabla_{\bd^{\ell}}F(\bd_k)\|_2 = \max_{i\in[L]}\|\nabla_{\bd^i}F(\bd_k)\|_2\},$ namely,
    \[
    \bd^{\ell}_* = \begin{cases}
        \frac{\nabla_{\bd^{\ell}}F(\bd_k)}{\|\nabla_{\bd^{\ell}}F(\bd_k)\|_2}, &\ell \in \{j : \|\nabla_{\bd^{j}}F(\bd_k)\|_2 = \max_{i\in[L]}\|\nabla_{\bd^i}F(\bd_k)\|_2\}\\
        0, &\mbox{otherwise.}
    \end{cases}
    \]
    
\end{proof}

\section{Non-Euclidean Gradient Descent on Quadratics}\label{sec:quadratics_proof}
To prove convergence of Non-Euclidean GD for the case of a sufficiently small step size
(Theorem \ref{th:quadratics1}), we follow standard arguments of smoothness and strong
convexity. The following definitions of smoothness and strong convexity are standard
generalizations from the Euclidean norm to an arbitrary norm.

\begin{definition}
We say that $\cL: \mathbb{R}^d \rightarrow \mathbb{R}$ is $(L, \norm{\cdot})$-smooth if
\begin{equation}
    \|\nabla \cL(\bw) - \nabla \cL(\bv)\|_* \leq L \|\bw - \bv\|
\end{equation}
for all $\bw, \bv \in \mathbb{R}^d$.
\end{definition}

\begin{definition}
We say that $\cL: \mathbb{R}^d \rightarrow \mathbb{R}$ is $(\mu, \norm{\cdot})$-strongly
convex if
\begin{equation} \label{eq:strongly_convex}
    \cL(\bv) \geq \cL(\bw) + \langle \nabla \cL(\bw), \bv - \bw \rangle + \frac{\mu}{2} \|\bv - \bw\|^2
\end{equation}
for all $\bw, \bv \in \mathbb{R}^d$.
\end{definition}

The following lemmas show that our quadratic $\cL(\bw) = \frac{1}{2} \bw^\top \mH \bw$ is
smooth and strongly convex.

\begin{lemma} \label{lem:quad_smooth}
The objective $\cL(\bw) = \frac{1}{2} \bw^T \mH \bw$ is $(L, \norm{\cdot})$-smooth with $L =
\sup_{\|\bz\| = 1} \bz^T \mH \bz$.
\end{lemma}

\begin{proof}
For any $\bw, \bv \in \mathbb{R}^d$, denote $\bd = (\bw - \bv)/\|\bw - \bv\|$. Then
\begin{align} \label{eq:smooth_inter}
    \frac{\|\nabla \cL(\bw) - \nabla \cL(\bv)\|_*}{\|\bw - \bv\|} = \frac{\|\mH\bw - \mH\bv\|_*}{\|\bw - \bv\|} = \| \mH \bd \|_* = \sup_{\|\bu_1\| = 1} \bu_1^{\top} \mH \bd &\leq \sup_{\|\bu_1\| = \|\bu_2\| = 1} \bu_1^{\top} \mH \bu_2,
\end{align}
where in the third equality we used the definition of dual norm.
Next we will prove that
$$\sup_{\|\bu_1\| = \|\bu_2\| = 1} \bu_1^{\top} \mH \bu_2 = \sup_{\|\bz\| = 1}
\bz^{\top} \mH \bz.$$ The $(\geq)$ direction is immediate since
\begin{equation}
    \sup_{\|\bu_1\| = \|\bu_2\| = 1} \bu_1^{\top} \mH \bu_2 \geq \sup_{\|\bz\| = 1} \bz^{\top} \mH \bz.
\end{equation}
To show the other direction, let
\begin{equation}
    (\bu_1^*, \bu_2^*) \in \argmax{\|\bu_1\| = \|\bu_2\| = 1} \bu_1^{\top} \mH \bu_2,
\end{equation}
and
\begin{equation}
    \bz^* \in \argmax{\|\bz\| = 1} \bz^{\top} \mH \bz.
\end{equation}
Note that these $\argmax{}$ operations make sense, since we are considering the maximum of
continuous functions on compact domains, which always achieve their supremum. Then
\begin{align*}
    (\bu_1^* - \bu_2^*)^{\top} \mH (\bu_1^* - \bu_2^*) &\geq 0 \\
    (\bu_1^*)^{\top} \mH \bu_1^* - 2 (\bu_1^*)^{\top} \mH \bu_2^* + (\bu_2^*)^{\top} \mH \bu_2^* &\geq 0 \\
    (\bu_1^*)^{\top} \mH \bu_1^* + (\bu_2^*)^{\top} \mH \bu_2^* &\geq 2 (\bu_1^*)^{\top} \mH \bu_2^* \\
    2 (\bz^*)^{\top} \mH \bz^* &\geq 2 (\bu_1^*)^{\top} \mH \bu_2^* \\
    (\bz^*)^{\top} \mH \bz^* &\geq (\bu_1^*)^{\top} \mH \bu_2^*,
\end{align*}
where the first inequality uses that $\mH$ is PSD, the second inequality uses that $\mH$
is symmetric, and the fourth inequality uses $(\bu_1^*)^{\top} \mH \bu_1^* \leq
(\bz^*)^{\top} \mH \bz^*$ and $(\bu_2^*)^{\top} \mH \bu_2^* \leq (\bz^*)^{\top} \mH
\bz^*$. This proves the $(\leq)$ direction, and proves the claim. Then Equation
\eqref{eq:smooth_inter} becomes
\begin{equation}
    \frac{\|\nabla \cL(\bw) - \nabla \cL(\bv)\|_*}{\|\bw - \bv\|} \leq \sup_{\|\bz\| = 1} \bz^{\top} \mH \bz,
\end{equation}
or
\begin{equation}
    \|\nabla \cL(\bw) - \nabla \cL(\bv)\|_* \leq \left( \sup_{\|\bz\| = 1} \bz^{\top} \mH \bz \right) \|\bw - \bv\|.
\end{equation}
\end{proof}

\begin{lemma} \label{lem:quad_sc}
The objective $\cL(\bw) = \frac{1}{2} \bw^T \mH \bw$ is $(\mu, \norm{\cdot})$-strongly convex
with $\mu = \inf_{\|\bv\| = 1} \bv^T \mH \bv$.
\end{lemma}

\begin{proof}
The strong convexity property
\begin{equation}
    \cL(\bv) \geq \cL(\bw) + \langle \nabla \cL(\bw), \bv - \bw \rangle + \frac{\mu}{2} \|\bv - \bw\|^2
\end{equation}
for our particular $\cL$ is equivalent to each of the following statements:
\begin{align}
    \frac{1}{2} \bv^{\top} \mH \bv &\geq \frac{1}{2} \bw^{\top} \mH \bw + (\bv - \bw)^{\top} \mH \bw + \frac{\mu}{2} \|\bv - \bw\|^2 \\
    \frac{1}{2} \bv^{\top} \mH \bv - \bv^{\top} \mH \bw + \frac{1}{2} \bw^{\top} \mH \bw &\geq \frac{\mu}{2} \|\bv - \bw\|^2 \\
    (\bv - \bw)^{\top} \mH (\bv - \bw) &\geq \mu \|\bv - \bw\|^2 \\
    \left( \frac{\bv - \bw}{\|\bv - \bw\|} \right)^{\top} \mH \frac{\bv - \bw}{\|\bv - \bw\|} &\geq \mu,
\end{align}
which is satisfied by $\mu = \inf_{\|\bv\| = 1} \bv^T \mH \bv$.
\end{proof}

\thquadraticsone*

\begin{proof}
To show convergence, we prove a generalization of the Polyak-\L ojasiewicz (PL)
property, then follow the standard analysis of gradient descent for smooth and PL
functions.

Lemma \ref{lem:quad_sc} implies that $\cL$ is $\mu$-strongly convex with $\mu =
\inf_{\|\bv\|=1} \bv^\top \mH \bv$. We also know that $\cL(\bw) \geq \cL_* := 0$, and
that this minimum is achieved at $\bw_* = \mathbf{0}$. So we apply
\eqref{eq:strongly_convex} with $\bv = \bw_*$ and any $\bw$:
\begin{align}
    \cL_* &\geq \cL(\bw) + \langle \nabla \cL(\bw), \bw_* - \bw \rangle + \frac{\mu}{2} \|\bw_* - \bw\|^2 \\
    &\geq \inf_{\bv} \left\{ \cL(\bw) + \langle \nabla \cL(\bw), \bv - \bw \rangle + \frac{\mu}{2} \|\bv - \bw\|^2 \right\}.
\end{align}
From \eqref{eq:gradmethod}, we know the $\inf$ above is minimized when $\bv = \bw -
\nicefrac{1}{\mu} \|\nabla \cL(\bw)\|_* (\nabla \cL(\bw))_*$. We also know that
$\cL(\bw) \geq \cL_* := 0$ for all $\bw$. So
\begin{align}
    \cL_* &\geq \cL(\bw) - \frac{1}{\mu} \|\nabla \cL(\bw)\|_* \langle \nabla \cL(\bw), (\nabla \cL(\bw))_* \rangle + \frac{1}{2 \mu} \|\nabla \cL(\bw)\|_*^2 \|(\nabla \cL(\bw))_*\|^2 \\
    &= \cL(\bw) - \frac{1}{\mu} \|\nabla \cL(\bw)\|_*^2 + \frac{1}{2 \mu} \|\nabla \cL(\bw)\|_*^2 \\
    &= \cL(\bw) - \frac{1}{2 \mu} \|\nabla \cL(\bw)\|_*^2,
\end{align}
so
\begin{equation} \label{eq:pl}
    \|\nabla \cL(\bw)\|_*^2 \geq 2 \mu (\cL(\bw) - \cL_*),
\end{equation}
which is the PL property we need.

Lemma \ref{lem:quad_smooth} implies that $\cL$ is $L$-smooth with $L = S$, so
\begin{align}
    \cL(\bw_{t+1}) &\leq \cL(\bw_t) + \langle \nabla \cL(\bw_t), \bw_{t+1} - \bw_t \rangle + \frac{S}{2} \|\bw_{t+1} - \bw_t\|^2 \\
    &\leq \cL(\bw_t) - \eta \|\nabla \cL(\bw_t)\|_* \langle \nabla \cL(\bw_t), (\nabla \cL(\bw_t))_* \rangle + \frac{S \eta^2 \|\nabla \cL(\bw_t)\|_*^2}{2} \|(\nabla \cL(\bw_t))_*\|^2 \\
    &\leq \cL(\bw_t) - \eta \|\nabla \cL(\bw_t)\|_*^2 + \frac{S \eta^2 \|\nabla \cL(\bw_t)\|_*^2}{2} \\
    &\leq \cL(\bw_t) - \eta \left( 1 - \frac{\eta S}{2} \right) \|\nabla \cL(\bw_t)\|_*^2 \\
    &\leq \cL(\bw_t) - 2 \mu \eta \left( 1 - \frac{\eta S}{2} \right) \left( \cL(\bw_t) - \cL_* \right),
\end{align}
where the last line uses the PL property from \eqref{eq:pl} and that $\eta < 2/S$.
Subtracting $\cL_*$ from both sides:
\begin{align}
    \cL(\bw_{t+1}) - \cL_* &\leq \left( 1 - 2 \mu \eta \left( 1 - \frac{\eta S}{2} \right) \right) \left( \cL(\bw_t) - \cL_* \right),
\end{align}
so that for all $t$,
\begin{equation}
    \cL(\bw_t) - \cL_* \leq \left( 1 - 2 \mu \eta \left( 1 - \frac{\eta S}{2} \right) \right)^t (\cL(\bw_0) - \cL_*).
\end{equation}
\end{proof}

The key to showing divergence when $\eta > 2/S$ (Theorem \ref{th:quadratics2}) is the
following lemma.

\leminvariant*

\begin{proof}
Since $\mH$ is symmetric and PSD, we have for any such $\bv$
\begin{align}
    (\bv - \hat{\bw})^{\intercal} \mH (\bv - \hat{\bw}) &\geq 0 \\
    \bv^{\intercal} \mH \bv - 2 \bv^{\intercal} \mH \hat{\bw} + \hat{\bw}^{\intercal} \mH \hat{\bw} &\geq 0 \\
    \bv^{\intercal} \mH \bv + \hat{\bw}^{\intercal} \mH \hat{\bw} &\geq 2 \bv^{\intercal} \mH \hat{\bw} \\
    2 \hat{\bw}^{\intercal} \mH \hat{\bw} &\geq 2 \bv^{\intercal} \mH \hat{\bw} \\
    \hat{\bw}^{\intercal} \mH \hat{\bw} &\geq \bv^{\intercal} \mH \hat{\bw},
\end{align}
where the fourth line uses that $\hat{\bw}^{\intercal} \mH \hat{\bw} \geq
\bv^{\intercal} \mH \bv$. Therefore
\begin{equation}
    \hat{\bw}\in\argmax{\|\bv\| = 1} \bv^{\intercal} \mH \hat{\bw}.
\end{equation}
Thus we may choose the dual gradient so that $(\mH \hat{\bw})_*=\hat{\bw}$.
\end{proof}

\thquadraticstwo*

\begin{proof}
Let $\bw_0 \in \text{span}(\hat{\bd})$ for some $\hat{\bd} \in \argmax{\|\bd\| = 1}
\bd^{\top} \mH \bd$, so $\hat{\bd} = \bw_0 / \|\bw_0\|$. We will show $\bw_t = (1 - \eta
S)^t \bw_0$ by induction on $t$, and  with the property of $\hat{\bd}$ from Lemma
\ref{lem:invariant}, the proof is essentially a direct calculation. The base case $t =0$
holds since $(1 - \eta S)^0 \bw_0 = \bw_0.$ By the induction hypothesis we have that
$\bw_t = \|\bw_0\|(1 - \eta S)^t \hat{\bd}$, and from the definition of gradient
descent,
\begin{align}
    \bw_{t+1}  &= \bw_t - \eta \left\| \mH \bw_t \right\|_* (\mH \bw_t)_* \\
    &= \|\bw_0\| (1 - \eta S)^t \hat{\bd} - \eta \|\bw_0\| |1 - \eta S|^t \left\| \mH \hat{\bd} \right\|_* (\|\bw_0\| (1 - \eta S)^t \mH \hat{\bd})_* \\
    &= \|\bw_0\| (1 - \eta S)^t \hat{\bd} - \eta \|\bw_0\| (1 - \eta S)^t \left\| \mH \hat{\bd} \right\|_* (\mH \hat{\bd})_* \\
    &= \|\bw_0\| (1 - \eta S)^t \hat{\bd} - \eta \|\bw_0\| (1 - \eta S)^t \left\| \mH \hat{\bd} \right\|_* \hat{\bd} \\
    &= \|\bw_0\| (1 - \eta S)^t \left( 1 - \eta \|\mH \hat{\bd}\|_* \right) \hat{\bd} \\
    &= (1 - \eta S)^{t+1} \bw_0.
\end{align}
where the second line uses the inductive hypothesis, the third line uses that the dual
map satisfies $(\lambda v)_* = \text{sign}(\lambda) v_*$ for any $\lambda \in \mathbb{R}$, the
fourth line uses~Lemma~\ref{lem:invariant}, and
the fifth line uses
\begin{equation}
    \|\mH \hat{\bd}\|_* = \sup_{\|\bv\| = 1} \bv^{\intercal} \mH \hat{\bd} = (\mH \hat{\bd})_*^{\intercal} \mH \hat{\bd} = \hat{\bd}^{\intercal} \mH \hat{\bd} = \sup_{\|\bv\| = 1} \bv^{\intercal} \mH \bv = S.
\end{equation}
\end{proof}

As an aside, we can also show that \algname{GD} will diverge for \textit{every}
initialization when $\eta$ is sufficiently large.

\begin{theorem}
Let $\cL(\bw) \coloneqq \frac{1}{2} \bw^\top \mH\bw$ for some $\mH \succ 0$.  For some norm $\| \cdot \|$, define the generalized sharpness $S^{\norm{\cdot}} := \max_{\|\bd\|\le 1} \bd^\top \mH \bd$. Then, if we run non-Euclidean \algname{GD} (Definition \ref{def:gradmethod}) on $\cL$, \algname{GD} will diverge for every initial point $\bw_0\neq0$ and any step-size $\eta > \nicefrac{2}{\mu}$.
\end{theorem}

\begin{proof}
Starting from the definition of gradient descent,
\begin{align}
    \bw_{t+1} &= \bw_t - \eta \| \mH \bw_t \|_* (\mH \bw_t)_* \\
    \mH \bw_{t+1} &= \mH \bw_t - \eta \| \mH \bw_t \|_* \mH (\mH \bw_t)_* \\
    \left\| \mH \bw_{t+1} \right\|_* &= \bigg\| \mH \bw_t - \eta \| \mH \bw_t \|_* \mH (\mH \bw_t)_* \bigg\|_* \\
    \left\| \mH \bw_{t+1} \right\|_* &\geq \eta \| \mH \bw_t \|_* \bigg\| \mH (\mH \bw_t)_* \bigg\|_* - \|\mH \bw_t\|_* \\
    \left\| \mH \bw_{t+1} \right\|_* &\geq \left( \eta \bigg\| \mH (\mH \bw_t)_* \bigg\|_* - 1 \right) \|\mH \bw_t\|_*. \label{eq:tight_div_inter}
\end{align}
We can bound the coefficient of $\eta$ as
\begin{equation} \label{eq:loose_div_inter}
    \bigg\| \mH (\mH \bw_t)_* \bigg\|_* \geq \inf_{\|\bv\| = 1} \|\mH \bv\|_* = \inf_{\|\bv\| = 1} \sup_{\|\bu\| = 1} \bu^{\intercal} \mH \bv \geq \inf_{\|\bv\| = 1} \bv^{\intercal} \mH \bv = \mu,
\end{equation}
so
\begin{align}
    \left\| \mH \bw_{t+1} \right\|_* &\geq \left( \eta \mu - 1 \right) \|\mH \bw_t\|_*,
\end{align}
and therefore
\begin{align}
    \left\| \mH \bw_t \right\|_* &\geq \left( \eta \mu - 1 \right)^t \|\mH \bw_0\|_*.
\end{align}
Since $\eta > 2/\mu \implies \eta \mu - 1 > 1$,  the parameter norm $\|\mH \bw_t\|_*$
increases exponentially, and GD diverges.
\end{proof}

\end{document}